\documentclass{article}

\usepackage[utf8]{inputenc}
\usepackage[sectionbib]{natbib}
\usepackage{natbib}
\usepackage{bm}
\usepackage{relsize}
\usepackage[a4paper]{geometry}

\usepackage{amssymb}

\usepackage{myAlgorithm}

\usepackage{lscape}
\usepackage{bm}
\usepackage{amsmath}
\usepackage{amsthm}
\usepackage[normalem]{ulem}
\theoremstyle{break}
\newtheorem{theorem}{Theorem}
\newtheorem{proposition}{Proposition}
\newtheorem{lemma}{Lemma}
\newtheorem{definition}{Definition}

\newtheorem{corollary}{Corollary}
\newtheorem{assumption}{Assumption}

\usepackage{graphicx}
\usepackage{todonotes}
\usepackage{pdflscape}
\usepackage{afterpage}
\usepackage[normalem]{ulem}

\usepackage{caption}

\DeclareMathOperator{\Det}{Det}

\DeclareMathOperator{\EX}{\mathbb{E}}
\DeclareMathOperator{\R}{\mathbb{R}}

\usepackage{amsmath}
\usepackage{bbold}
\usepackage{bm}
\usepackage{amssymb}
\usepackage{amsthm}
\usepackage{tikz}
\usetikzlibrary{arrows.meta}
\usepackage{marginnote}

\usepackage{hyperref}
\hypersetup{
    colorlinks=true,
    linkcolor=blue,
    filecolor=green,  
    citecolor =blue,  
    urlcolor=cyan,
}
\usepackage[capitalise]{cleveref}
\usepackage{subfig}

\DeclareMathOperator{\DPP}{\mathrm{DPP}}

\DeclareMathOperator{\VS}{\mathrm{CVS}}

\DeclareMathOperator{\F}{\mathcal{F}}
\DeclareMathOperator{\X}{\mathcal{X}}
\DeclareMathOperator{\Ltwo}{\mathbb{L}_{2}( \omega)}

\DeclareMathOperator{\Ns}{\mathbb{N}^{*}}

\DeclareMathOperator*{\argmin}{arg\,min}
\def\UN{\:\mathcal{U}_N}

\newcommand{\rev}[1]{\textcolor{black}{ #1}}

\title{Signal reconstruction using determinantal sampling}

\author{
Ayoub Belhadji$^{\dagger}$\footnote{Corresponding author: \href{mailto:ayoub.belhadji@gmail.com}{ayoub.belhadji@gmail.com}}, R\'emi Bardenet$^{\ddagger}$, Pierre Chainais$^{\ddagger}$\\
\small $^{\dagger}$ Univ Lyon, ENS de Lyon, Inria, CNRS, UCBL, LIP UMR 5668, Lyon, France
\\
\small $^{\ddagger}$ Univ. Lille, CNRS, Centrale Lille, UMR 9189 - CRIStAL, 59651 Villeneuve d’Ascq, France \\
}

\begin{document}

\maketitle

\begin{abstract}
We study the approximation of a square-integrable function from a finite number of evaluations on a random set of  nodes according to a well-chosen distribution. 
This is particularly relevant when the function is assumed to belong to a reproducing kernel Hilbert space (RKHS). This work proposes to combine several natural finite-dimensional approximations based on two possible probability distributions of nodes.
These distributions are related to determinantal point processes, and use the kernel of the RKHS to favor RKHS-adapted regularity in the random design.
While previous work on determinantal sampling relied on the RKHS norm, we prove mean-square guarantees in $L^2$ norm.
We show that determinantal point processes and mixtures thereof can yield fast convergence rates. Our results also shed light on how the rate changes as more smoothness is assumed, a phenomenon known as superconvergence. 
Besides, determinantal sampling improves on i.i.d. sampling from the Christoffel function which is standard in the literature.
More importantly, determinantal sampling guarantees the so-called instance optimality property for a smaller number of function evaluations than i.i.d. sampling. 









\end{abstract}

\maketitle

\begin{keywords} Christoffel sampling; instance optimality property; finite-dimensional approximations; determinantal point processes; reproducing kernel Hilbert spaces.
\end{keywords}



\section{Introduction}
The problem of reconstructing a continuous signal from a set of discrete samples is a fundamental question of sampling theory. It has stimulated a considerable literature. 
This problem consists in approximating an unknown function $f$ by a surrogate $\hat{f}$ knowing a discrete set of evaluations of $f$. 
The Whittaker-Shannon-Kotel'nikov (WSK) sampling theorem is arguably the most emblematic result in the field. 
It can be seen as an exact interpolation result from periodic samples for functions that are band-limited in the Fourier domain \citep{Whi28,Sha49,Kot06}. This theorem has been extended to functions that are band-limited with respect to other transforms than Fourier including, for instance, Sturm-Liouville transforms \citep{Kra59,Cam64}, Laguerre transforms \citep{Jer76} and Jacobi transforms \citep{KoWa90}.

Reproducing kernel Hilbert spaces (RKHSs) and sampling problems have a long common history. 
Tracing back to \citep{Aro43,Aro50}, one possible interpretation of the definition of an RKHS $ \mathcal{F}$ of kernel $k$ is that any signal $f\in\mathcal{F}$ is a limit of weighted sums of kernel translates $k(x,\cdot)$.
In a seminal paper, \cite{Yao67} derived sufficient conditions on a configuration  of \emph{nodes} $(x_{i})_{i \in \mathcal{I}}$ to achieve exact reconstruction of any element $f$ of the RKHS $\mathcal{F}$, i.e., to ensure that uniformly, 
\begin{equation}
  \label{eq:exact_reconstruction_kernel_formula}
  f = \sum\limits_{i \in \mathcal{I}} f(x_{i}) k(x_{i},\cdot).
\end{equation}
Yao proved that the spaces of band-limited signals in the Fourier, Bessel, or cosine domains are all RKHSs.
In particular, Yao's result generalizes the WSK theorem.  
Weaker sufficient conditions than Yao's for \eqref{eq:exact_reconstruction_kernel_formula} to hold have been studied, see e.g. \citep{NaWa91}.
%
However the WSK sampling theorem and its extensions to RKHSs remain \emph{asymptotic}: an infinite number of samples is required in order to guarantee the exact reconstruction of the function. 
In real applications, only a finite number of evaluations $f(x_{1}), \dots, f(x_{N})$ at nodes $x_{1}, \dots, x_{N} \in \mathcal{X}$ is available. 
Therefore \emph{non-asymptotic} guarantees on reconstructions from $N$ samples are necessary. They are the purpose of the present work.

Interestingly, the non-asymptotic problem of the reconstruction from a finite number of evaluations arose first historically, as mentioned by \cite{Hig85}. 
Indeed, an early form of sampling theorem may be traced back to an interpolation scheme due to \cite{Pou08}; see \citep{BuSt92} for a historical account. 
Non-asymptotic sampling results have resurged in popularity recently \citep{CoMi17,Bac17,AvKaMuVeZa19}. 
In those works, the nodes $x_{1}, \dots, x_N$ were taken to be 
independent draws from a particular probability distribution. 
The latter is closely related to extensions of the so-called \emph{Christoffel function}, a classical tool in the theory of orthogonal polynomials \citep{Nev86}. 
These configurations of independent particles were shown to yield approximations that satisfy the \emph{instance optimality property} (IOP). IOP provides a desirable multiplicative error bound that guarantees exact recovery on a finite-dimensional subspace, with an almost optimal number of nodes.
On the other hand, alternative designs have been proposed to achieve optimal reconstruction, using boosting \citep{HaNoPe22} or sparsification techniques \citep{KrUl21,DoCo22,DoKrUl23,ChDo23}.











This work presents a general approach for signal reconstruction, based on \emph{non-independent} random nodes. 
The key idea is to use random nodes sampled from mixtures of determinantal point processes (DPP).
DPPs are distributions over configurations of points that encode repulsion in the form of a kernel: connecting the DPP's kernel to the RKHS kernel makes the nodes repel, with a repulsion related to the smoothness of the target function. 
DPPs were introduced in Odile Macchi's 1972 thesis --recently translated and reprinted as \citep{Mac17}-- as models for detection times in fermionic optics. Since then, they have been thoroughly studied in random matrix theory \citep{Joh05}, and have more recently been adopted in machine learning \citep{KuTa12}, spatial statistics \citep{LaMoRu15}, and Monte Carlo methods \citep{BaHa20}.


%
%
The proposed non-asymptotic guarantees for function reconstruction with DPPs are motivated by previous works on numerical integration \citep{BeBaCh19,BeBaCh20}. 
In these works, we measured reconstruction performance using the RKHS norm. Then a rather strong smoothness assumption appeared necessary, ensuring that the function belong to a particular strict subspace of the RKHS. 
Now motivated by applications in signal reconstruction, and to overcome the  smoothness assumption of previous works, an alternative analysis is proposed that replaces the RKHS norm by the $L_2$ norm.
In particular, we study various finite-dimensional approximations of a function living in an RKHS, such as the least-squares approximation with respect to the $L_2$ norm. 
We show that its mean square error converges to zero at a rate that depends on the eigenvalues of the RKHS kernel.
Moreover, we show that the convergence is faster for functions living in a certain low-dimensional subspace of the RKHS. 
This sheds light on the phenomenon of \emph{super-convergence} observed in the literature of kernel-based approximations \citep{Sch18}, and relates to the question of \emph{completeness} of DPPs \citep{Lyo14}. 
Yet, in spite of its remarkable theoretical properties, the least-squares approximation with respect to the $L_2$ norm cannot be evaluated numerically by using a finite number of evaluations of the target function only.
For this reason, we also investigate more practical approximations based on two particular \emph{transforms}, i.e. linear operators from the RKHS to finite-dimensional vector spaces, built through a compilation of quadrature rules. 
In particular, we prove that the aforementioned \emph{instance optimality property} holds for a particular transform-based approximation using DPPs, even with a minimal sampling budget. 
Numerical experiments in dimension one as well as on the hypersphere validate our results experimentally. They show the very good empirical efficiency of the proposed approximations. Even though our bounds guarantee good performance, the numerical performance actually exceeds the expectations, which shows as well that there is still some room for improvement of our theoretical results.

This article is organized as follows. Section~\ref{sec:sampling_rkhs} contains notations and definitions, as well as a brief review of previous work on the reconstruction of RKHS functions based on discrete samples. 
In~\Cref{sec:new_results}, we present our main results. 
In Section~\ref{s:numsims}, we illustrate our results and compare to related work using numerical simulations.
In Sections~\ref{s:discussion} and \ref{s:conclusion}, we discuss a number of open questions and conclude.

\section{Sampling and reconstruction in RKHSs}
\label{sec:sampling_rkhs}

In Section~\ref{sec:notation}, for ease of reference, we introduce some basic notation and assumptions to be used throughout the paper. 
In~\Cref{sec:smoothness_frac}, we define fractional subspaces, which correspond to increasing levels of smoothness within an RKHS.
In~\Cref{sec:non-asymptotic}, we recall the finite-dimensional approximations that we will work with in the rest of the article. To give historical context, \Cref{sec:design_finite_dim} and~\Cref{s:DPPs_th_old} are devoted to related works on the topic. 
The contents of these two sections may be skipped in a first read except for~\Cref{s:DPPs} which is necessary for the statement of the main results presented in~\Cref{sec:new_results}.

\subsection{Notations and assumptions}
\label{sec:notation}

\subsubsection{Kernels and RKHSs}
Let $\mathcal{X}$ be a metric set, equipped with a positive Borel measure $\omega$. 
Let $k : \mathcal{X} \times \mathcal{X} \rightarrow \mathbb{R}$ be a symmetric, positive definite kernel.
Consider the inner product defined on the space $\mathcal F_0$ of finite linear combinations of kernel translates by
\begin{equation}
  \left\langle \sum_{i \in [m]} a_i k(x_i,\cdot), \sum_{j \in [n]} b_j k(y_j, \cdot) \right\rangle_{\mathcal{F}} = \sum_{i \in [m]} \sum_{j \in [n]} a_i b_j k(x_i, y_j),
\end{equation}
where $a_1, b_1, \dots \in \mathbb{R}$ and $x_1, y_1, \dots \in\mathcal{X}$. Here and in the sequel $[N]$ will denote the set of indices ranging from $1$ to $N$.
The completion $\mathcal{F}$ of $\mathcal{F}_0$ for $\langle \cdot,\cdot \rangle_{\mathcal{F}}$ is the so-called reproducing kernel Hilbert space of kernel $k$.
By the Moore-Aronszajn theorem, it is the unique Hilbert space of functions that satisfies the reproducing property, i.e., for all $f\in\mathcal F$ and $x\in\mathcal{X}$, $f(x) = \langle f, k(x,\cdot)\rangle_\mathcal{F}$; see e.g. \citep{BeTh11} for a general reference.

The properties of a function in $\mathcal F$ are often described in terms of the spectral characteristics of the integral operator of kernel $k$, at the price of some additional technicalities. 
Because we shall need a so-called \emph{Mercer decomposition}, we follow here the precise reference paper of \cite{StSc12}.
A convenient way to diagonalize the operator $\bm{\Sigma}$ defined by 
\begin{equation}
  \label{eq:integral_operator}
  \bm{\Sigma}f(x) = \int k(x,y) f(y) \mathrm{d}\omega(y)
\end{equation}
is to make sure it defines a compact operator from $\Ltwo$ to itself.
A practical sufficient condition for this compactness to hold is that $\int_{\mathcal{X}} k(x,x) \mathrm{d}\omega(x)<+\infty$; see \cite[Lemma 2.3]{StSc12}. 
This is in particular implied by our Assumptions~\ref{a:compactness} and \ref{a:compactness_omega}, whose additional strength we shall later require.

\begin{assumption}
  \label{a:compactness}
  The diagonal of the kernel $x\mapsto k(x,x)$ is bounded in $\mathcal{X}$.
\end{assumption}
We note that the positive definiteness of $k$ implies that for all $x,y$, 
$$
  k(x,x)k(y,y) - k(x,y)^2 \geq 0,
$$
so that under Assumption~\ref{a:compactness}, $x,y\mapsto k(x,y)$ is bounded in $\mathcal{X}\times\mathcal{X}$.

\begin{assumption}
  \label{a:compactness_omega}
  The measure $\omega$ is finite, i.e. $\int_{\mathcal{X}} \mathrm{d}\omega(x) <+\infty$.
\end{assumption}

Under Assumptions~\ref{a:compactness} and \ref{a:compactness_omega}, $\bm{\Sigma}$ in \eqref{eq:integral_operator} is a self-adjoint, non-negative, compact operator.
The spectral theorem, see e.g. Chapter 6 in \citep{Bre10}, guarantees the existence of an orthonormal basis $(\widetilde{e}_{m})_{m \in \mathbb{N}^{*}}$ of $\Ltwo$ and a family of non-increasing non-negative scalars $(\sigma_{m})_{m \in \mathbb{N}^{*}}$ such that $\bm{\Sigma} \widetilde{e}_{m} = \sigma_{m}\widetilde{e}_{m}$ for $m \in \mathbb{N}^{*}$. 
We now make extra assumptions to reach a Mercer decomposition.

\begin{assumption}
  \label{a:continuity_and_density}
  The kernel $k$ is continuous, $\omega$ has full support, and \rev{all eigenvalues of $\bm{\Sigma}$ are positive}. 
\end{assumption}
Under Assumptions~\ref{a:compactness} to \ref{a:continuity_and_density}, it is easy to see that $\bm{\Sigma}$ maps to continuous functions, so that $\widetilde{e}_m = {\sigma}_m^{-1} \bm{\Sigma}\widetilde{e}_m$ has a continuous representative $e_m$ in its equivalence class in $\Ltwo$. 
Letting $e_{m}^{\mathcal{F}}:= \sqrt{\sigma_{m}} e_{m}$, the family $(e_{m}^{\mathcal{F}})_{m \in \mathbb{N}^{*}}$ is then an orthonormal basis of $\mathcal{F}$, and $\mathcal{F}$ is isometrically isomorphic to a dense subspace of $\Ltwo$ through 
\begin{equation}
  \label{e:isometry}
  \sum a_m e_m^\mathcal{F} \mapsto \sum \sigma_m^{1/2} a_m \widetilde{e}_m;
\end{equation}
see \citep[Theorem 2.11, Lemma 2.12, and Corollary 3.5]{StSc12}. 
In particular, it is customary to (abusively) write $\mathcal{F} = \bm{\Sigma}^{1/2}\Ltwo$ and, for $f\in\mathcal{F}$, 
\begin{equation}
  \label{eq:def_F_norm}
  \|f\|_{\F}^{2} = \Vert\bm{\Sigma}^{-1/2}f\Vert_\omega^2 = \sum\limits_{m \in \mathbb{N}^{*}} \frac{\langle f,e_{m} \rangle_{\omega}^2}{\sigma_m} <+\infty.
\end{equation}
Finally, by \cite[Corollary 3.5]{StSc12}, for any $x\in\mathcal{X}$, we can write the so-called \emph{Mercer} decomposition,
\begin{equation}\label{eq:mercer_infinite_sum}
  k(x,y) = \langle k(x,\cdot), k(y,\cdot) \rangle_{\mathcal{F}} = \sum\limits_{m \in \mathbb{N}^{*}} e_{m}^{\F}(x)e_{m}^{\F}(y) = \sum\limits_{m \in \mathbb{N}^{*}} \sigma_{m} e_{m}(x)e_{m}(y),
\end{equation}  
where the convergence is uniform in $y\in\mathcal{X}$.
Assumptions~\ref{a:compactness}, \ref{a:compactness_omega}, and \ref{a:continuity_and_density} hold for the rest of this article.

\subsubsection{Different kernels}
\label{s:different_kernels}
\rev{We shall need a handful of kernels related to $k$ via a modification of its Mercer decomposition.
For $\nu \in \{1,2\}$, let the kernel $k_{\nu}$ be defined by}
\rev{\begin{equation} k_{\nu}(x,y) = \sum\limits_{m \in \mathbb{N}^{*}} \sigma_{m}^{\nu} e_{m}(x)e_{m}(y),\end{equation}
where for each $x\in\mathcal{X}$, the convergence is uniform in $y\in\mathcal{X}$.
The corresponding Gram matrix is denoted by $\bm{K}_{\nu}(\bm{x}) := (k_{\nu}(x_i,x_j))_{i,j \in [N]}$. 
Similarly, for $\nu \in \{0,1,2\}$ and 
$T\subset \mathbb{N}^*$ a finite subset of indices, we consider the truncated kernel $k_{\nu,T}$ defined by}
\rev{\begin{equation} 
  \label{e:different_truncated_kernels}
  k_{\nu,T}(x,y) = \sum\limits_{m \in T} \sigma_{m}^{\nu}e_{m}(x)e_{m}(y).
\end{equation}
We will often use $k_{\nu, T}$ with $T=[N]=\{1, \dots, N\}$, in which case we write $k_{\nu, N} = k_{\nu, [N]}$, with the corresponding Gram matrix denoted by $\bm{K}_{\nu,N}(\bm{x})$.} 
%

\rev{Finally, for a set of points $\bm{x} = \{x_1, \dots, x_N\} \subset \mathcal{X}$, define the finite-dimensional subspace 
\begin{equation}\label{eq:def-set-Kx}
  \mathcal{K}(\bm{x})= \mathrm{Span}\{k(x_1,.), \dots, k(x_N,.) \} \subset \mathcal{H}.
\end{equation}  
Projections onto subspaces like $\mathcal{K}(\bm{x})$ will play an important role, so we introduce our last notation.
Given a Hilbert semi-norm $\|.\|$ defined on $\mathbb{L}_{2}(\omega)$, and a closed linear space $\mathcal{G}$ of $\mathbb{L}_{2}(\omega)$, let $\Pi_{\mathcal{G}}^{\|.\|}$ be the $\|.\|$-orthogonal projector onto $\mathcal{G}$. 
}

\subsubsection{Increasing levels of smoothness}
\label{sec:smoothness_frac}
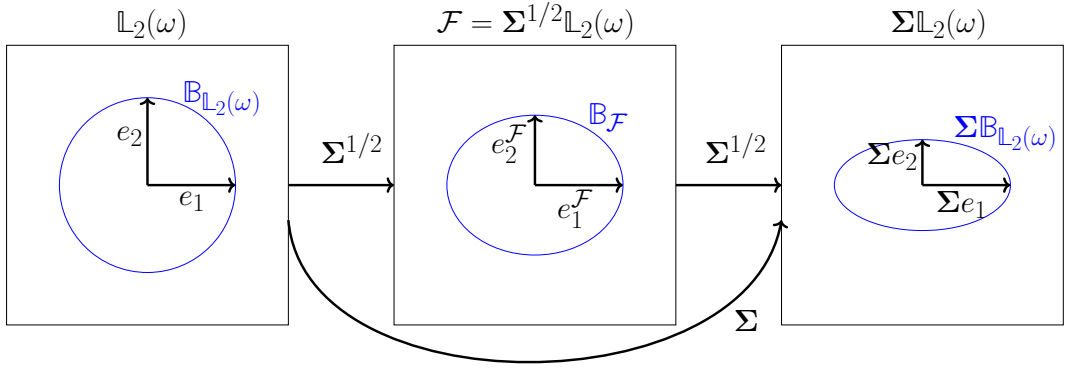
\begin{figure*}
    \centering
    \resizebox{0.95\textwidth}{!}{\makeatletter


\begin{tikzpicture}[]
\coordinate (A) at (1.4,1,0);
\coordinate (B) at (6.7,1,0);
\coordinate (B2) at (4.1,1,0);
\coordinate (C) at (6.7,-3,0);
\coordinate (D) at (8,-0.3,0);
\coordinate (E) at (8,-1.7,0);

\coordinate (F) at (4,-1,0);

\coordinate (E_1_L2_1) at (0, 0);
\coordinate (E_1_L2_2) at (2.5, 0);

\coordinate (E_2_L2_1) at (0, 0);
\coordinate (E_2_L2_2) at (0, 2.5);

\coordinate (E_1_F_1) at (9+2, 0);
\coordinate (E_1_F_2) at (11.5+2, 0);

\coordinate (E_2_F_1) at (9+2, 0);
\coordinate (E_2_F_2) at (9+2, 2);

\coordinate (E_1_FF_1) at (18+4, 0);
\coordinate (E_1_FF_2) at (20.5+4, 0);

\coordinate (E_2_FF_1) at (18+4, 0);
\coordinate (E_2_FF_2) at (18+4, 1.3);

\coordinate (L2_F_1) at (4, 0);
\coordinate (L2_F_2) at (7, 0);
\coordinate (F_FF_1) at (15, 0);
\coordinate (F_FF_2) at (18, 0);

\coordinate (L2_FF) at (12, 5);

\draw [->,line width=2pt] (L2_F_1) -- (L2_F_2);
\draw [->,line width=2pt] (F_FF_1) -- (F_FF_2);

\draw [->,line width=2pt] (4, -1) to [out=-85,in=-100] (18, -1);

\draw (5.8, 0.5) node[above ] {\Huge $\bm{\Sigma}^{1/2}$};
\draw (16.7, 0.5) node[above] {\Huge $\bm{\Sigma}^{1/2}$};
\draw (17, -4.3) node[above] {\Huge $\bm{\Sigma}$};

\draw (-4,-4) rectangle (4cm,4cm);
\draw[blue] (0,0) circle (2.5cm and 2.5cm);

\draw (7,-4) rectangle (15cm,4cm);
\draw[blue] (11,0) ellipse (2.5cm and 2cm);

\draw (18,-4) rectangle (26cm,4cm);
\draw[blue] (22,0) ellipse (2.5cm and 1.3cm);

\draw [->,line width=2pt, black] (E_1_L2_1) -- (E_1_L2_2);
\draw [->,line width=2pt, black] (E_2_L2_1) -- (E_2_L2_2);

\draw [->,line width=2pt, black] (E_1_F_1) -- (E_1_F_2);
\draw [->,line width=2pt, black] (E_2_F_1) -- (E_2_F_2);

\draw [->,line width=2pt, black] (E_1_FF_1) -- (E_1_FF_2);
\draw [->,line width=2pt, black] (E_2_FF_1) -- (E_2_FF_2);





\begin{scope}
  \draw (1.3,4,0) node[above left = 0.5mm] {\Huge $\mathbb{L}_{2}(\omega$)};

  \draw (14,4,0) node[above left = 0.5mm] {\Huge $ \mathcal{F} = \bm{\Sigma}^{1/2}\mathbb{L}_{2}(\omega$)};
  \draw (24,4,0) node[above left = 0.5mm] {\Huge $ \bm{\Sigma}\mathbb{L}_{2}(\omega$)};

 \draw (3.4,1.8,0) node[above left = 0.5mm] {\Huge \textcolor{blue}{$\mathbb{B}_{\mathbb{L}_{2}(\omega)}$}};

  \draw (13.8,1.4,0) node[above left = 0.5mm] {\Huge \textcolor{blue}{$\mathbb{B}_{\mathcal{F}}$}};
  \draw (26,0.8,0) node[above left = 0.5mm] {\Huge \textcolor{blue}{$\bm{\Sigma}\mathbb{B}_{\mathbb{L}_{2}(\omega)}$}};

  \draw (1.8, -0.9) node[above left = 0.5mm] {\Huge \color{black}$e_{1}$\color{black}};
  \draw (0, 1.0) node[above left = 0.1mm] {\Huge \color{black}$e_{2}$\color{black}};

  \draw (12.8, -1.4) node[above left = 0.5mm] {\Huge \color{black}$e_{1}^{\mathcal{F}}$\color{black}};
  \draw (10.9, 0.5) node[above left = 0.1mm] {\Huge \color{black}$e_{2}^{\mathcal{F}}$\color{black}};

\draw (24, -1.1) node[above left = 0.5mm] {\Huge \color{black}$\bm{\Sigma}e_1$\color{black}};
  \draw (22, 0.3) node[above left = 0.1mm] {\Huge \color{black}$\bm{\Sigma}e_2$\color{black}};




\end{scope}
\end{tikzpicture}
}
    \caption{A schematic diagram illustrating the relationship between the unit ball $\mathbb{B}_{\mathbb{L}_{2}(\omega)}$ of $\mathbb{L}_{2}(\omega)$, the unit ball $\mathbb{B}_{\F}$ of $\F$, and the image of $\mathbb{B}_{\mathbb{L}_{2}(\omega)}$ by the integration operator $\bm{\Sigma}$. 
    \label{fig:ellipsoids}
    }
\end{figure*}

The fact that $\mathcal{F}  = \bm{\Sigma}^{1/2} \mathbb{L}_{2}(\omega)$ illustrates the smoothing nature of the integral operator $\bm{\Sigma}$.
Similarly, for $r \geq 0$, an element $g \in \mathbb{L}_{2}(\omega)$ belongs to $\bm{\Sigma}^{r+1/2}\mathbb{L}_{2}(\omega)$ if and only if 
\begin{equation}
  \sum_{m \in \mathbb{N}^{*}} \frac{\langle g,e_{m} \rangle_{\omega}^2}{\sigma_{m}^{2r+1}}<+\infty.
\end{equation}
Moreover, since $(\sigma_m)$ is a non-increasing sequence of positive real numbers, we have
\begin{equation}
 r \geq r' \geq 0 \implies \bm{\Sigma}^{r+1/2}\mathbb{L}_{2}(\omega) \subset \bm{\Sigma}^{r'+1/2}\mathbb{L}_{2}(\omega).
\end{equation}
In particular, for $r\geq 0$, $\bm{\Sigma}^{r+1/2}\mathbb{L}_{2}(\omega) \subset \F$: these subspaces define increasing levels of regularity in the RKHS $\F$.
Figure~\ref{fig:ellipsoids} illustrates this hierarchy of functional spaces.

\subsection{Finite-dimensional approximations in RKHSs}
\label{sec:non-asymptotic}

The reconstruction of a function $f$ that belongs to an RKHS $\mathcal{F}$ based on its evaluations over a finite set of $N$ nodes $x_1, \dots x_N \in \mathcal{X}$ can be achieved by different families of finite-dimensional approximations. The following approximations will be of interest for the rest of the article.

\subsubsection{Approximations based on mixtures of kernel translates}

By definition of an RKHS, a natural choice of approximation is a weighted sum of kernel functions. 
Formally, the objective is thus to build a set of nodes $(x_i)$, a.k.a. a \emph{design}, and weights $(w_i)$, such that a suitable norm of the residual
\begin{equation}\label{eq:error_anorm}
  f- \sum_{i \in [N]}w_{i}k(x_{i},\cdot)
\end{equation}
is small. 
Assuming the design $\bm{x}=(x_i)_{i\in[N]}$ is fixed, we now consider two possible choices for this norm that induce different sets of \rev{optimal} weights.

Minimizing the $\Ltwo$ norm of the residual \eqref{eq:error_anorm} yields the classical least-squares (LS) approximation
\rev{\begin{equation}\label{eq:interpolation_omega_opt_problem}
  \hat{f}_{\mathrm{LS},\bm{x}} = \rev{\Pi_{\mathcal{K}(\bm{x})}^{\|.\|_{\omega}}f}  = \sum_{i \in [N]}\hat{w}_{i}k(x_{i},.)
\end{equation}
where the subspace ${\mathcal K}(\bm{x})$ has been introduced in \eqref{eq:def-set-Kx}, and
\begin{equation}
  \label{e:w_min}
  \hat{\bm{w}} = \arg\min_{\rev{\bm{w} \in \R^N}}\|f - \sum_{i \in [N]}w_{i}k(x_{i},.)\|_{\omega}.
\end{equation}}
The least-squares approximation typically enjoys strong theoretical properties. 
However, we shall see in \Cref{sec:L2_optimal_approximation} that computing the weights \eqref{e:w_min} requires the evaluation of $\bm{\Sigma}f$ at the nodes $x_{1}, \dots, x_N$. 
This makes the method impractical, and calls for more tractable approximations.

The RKHS norm of the residual is a natural alternative objective to minimize, and yields the so-called \emph{optimal kernel approximation} (OKA) 
\begin{equation}
  \label{eq:interpolation_opt_problem}
  \hat{f}_{\mathrm{OKA},\bm{x}} = \rev{\Pi_{\mathcal{K}(\bm{x})}^{\|.\|_{\mathcal{F}}}f} = \arg\min\limits_{\hat{f} \in \rev{\mathcal{K}(\bm{x})}} \|f - \hat{f}\|_{\F}.
\end{equation}
We shall sometimes write $\hat{f}_{\mathrm{OKA}}$ for $\hat{f}_{\mathrm{OKA},\bm{x}}$ when the design $\bm{x}$ is clear from the context. 
Let us observe that the approximation \eqref{eq:interpolation_opt_problem} is uniquely defined if the Gram matrix $ \bm{K}_1(\bm{x}) = (k(x_{i},x_{i'}))_{i,i' \in [N]}$ is non-singular, the optimal vector of weights being 
\begin{equation}
  \label{e:OKA_weights}
  \widehat{\bm{w}}_{\mathrm{OKA},\bm{x}} := \bm{K}_1(\bm{x})^{-1}f(\bm{x}),
\end{equation} 
as can be seen from expanding the norm in \eqref{eq:interpolation_opt_problem} and using the reproducing property. 
Note that, unlike the least-squares approximation, computing the weights $\widehat{\bm{w}}_{\mathrm{OKA},\bm{x}}$ only depends on being able to evaluate $f$ at the nodes. 
Moreover, OKA comes with a remarkable interpolation property: for $i \in [N]$, $\hat{f}_{\mathrm{OKA},\bm{x}}(x_i) = f(x_i)$.

%
\subsubsection{Approximations living in eigenspaces}
\label{sec:eigenspace_approximations}
Another way to recover a continuous signal $f \in \F$ from its evaluations $f(x_1), \dots, f(x_N)$ at fixed nodes is to seek an approximation that belongs to the eigenspace $\mathcal{E}_M = \mathrm{Span}(e_1, \dots, e_M)$ for some\footnote{
  \rev{The parameter $M$ is typically thought as smaller than or equal to $N$. Since some of our results allows for a general $M$, we prefer at this stage to let $M$ be any integer unless otherwise specified.}
} 
$M \in \mathbb{N}^{*}$.
We consider two such approaches. 


\paragraph{Transform-based approximations}
Consider the projection of $f$ onto $\mathcal{E}_M$,
\begin{equation}
  \label{eq:eigen_approximation}
  f_{M}:= \rev{\Pi_{\mathcal{E}_M}^{\|.\|_{\omega}} f} = \sum_{m \in [M]} \langle f, e_m \rangle_{\omega} e_m  .
\end{equation}
Note that, for $m \in [M]$, $\langle f, e_m \rangle_{\omega}$ is the integral $I_{m}(f):= \int_{\mathcal{X}}f(x)e_{m}(x) \mathrm{d}\omega(x)$. 
Since we want to assume the availability of the evaluation of $f$ only at the nodes, we consider approximating each integral $I_{m}(f)$ by a weighted sum
\begin{equation}
  \label{e:quadrature}
	\hat{I}_{m}(f) := \sum_{i \in [N]}\alpha_{m,i} f(x_i),
\end{equation}
where $\alpha_{m,i}$ are quadrature weights to be discussed shortly. 
The collection of operators $\hat{I}_{m}$ form a \emph{transform} $\Phi$, defined as $\Phi(f) =  (\hat{I}_m(f))_{m\in[M]}\in\mathbb{R}^{M}$. 
The resulting approximation $\hat{f}_{\Phi}$ will be called a \emph{transform-based approximation} of $f$. Note that
\begin{equation}\label{def:hat_phi_x}
	\hat{f}_{\Phi}:= \sum_{m \in [M]} \hat{I}_{m}(f) e_{m} = \sum_{m\in [M]} (f(\bm{x})^{\mathrm{T}}\bm{\alpha}_{m}) e_m = \sum_{i\in [N]} \left( \sum_{m\in [M]} \alpha_{m,i} e_m \right) f(x_i),
\end{equation}
\rev{which leads to the implicit construction of a family of $N$ functions, here in the parentheses, that define a {\em transform} with coefficients given by the samples $f(x_i)$ themselves.}


For the construction of the quadrature weights $(\alpha_{m,i})$ in \eqref{e:quadrature}, this work focuses on \emph{interpolative quadrature rules} \citep{Lar72}. Given a kernel $\kappa : \mathcal{X} \times \mathcal{X} \rightarrow \mathbb{R}$, 
the weights $\alpha_{m,i}$ of an interpolative quadrature are chosen so that 
\begin{equation}
  \label{eq:interpolative_kernel_quadrature}
  \forall i \in [N], m \in [\rev{N}], \:\: \hat{I}_{m}(\kappa(x_i,\cdot)) = I_{m}(\kappa(x_i,\cdot)).
\end{equation}
Under the assumption that the functions $\kappa(x_{1},\cdot),\dots,\kappa(x_{N},\cdot)$ are linearly independent, the sequence $(\alpha_{m,i})_{(m,i) \in [N] \times [N]}$ is uniquely defined by~\eqref{eq:interpolative_kernel_quadrature}. 
We now single out two such interpolative quadratures that correspond to two choices for $\kappa$\rev{, namely either $k$ or the truncated kernel $k_{0,N}$ defined in \cref{s:different_kernels}}

\rev{
 The first interpolative quadrature we consider is what we call \emph{optimal kernel quadrature} (OKQ). In this case, we take $\kappa = k$. Under the assumption that the matrix $\bm{K}_1(\bm{x})$ is non-singular, \eqref{eq:interpolative_kernel_quadrature} is satisfied if and only if the weights $\bm{\alpha}_{m}=(\alpha_{m,i})_{i \in [N]}$ in \eqref{e:quadrature} are taken to be
  \begin{equation}
    \label{eq:OKQ_alpha_n}
  \forall m \in [N],\:\:  \bm{\alpha}_{m}^{\mathrm{OKQ}} = \sigma_{m}\bm{K}_1(\bm{x})^{-1} e_{m}(\bm{x}).
  \end{equation}
 Indeed, we have $I_{m}(k(x,\cdot)) = \sigma_{m}e_{m}(x)$ for $x \in \mathcal{X}$. 
  While \eqref{eq:interpolative_kernel_quadrature} does not necessarily hold for $m \geq N+1$, we can still define $\bm{\alpha}_{m}^{\mathrm{OKQ}}$ by extending \eqref{eq:OKQ_alpha_n} to $m \geq N+1$.}
 The vector $\bm{\alpha}_{m}^{\mathrm{OKQ}}$ is then nothing but \rev{the vector obtained by taking $f(x) = \sigma_{m}e_{m}(x)$ in \eqref{e:OKA_weights}.}
%
In the rest of the article, the corresponding transform will be called the \emph{optimal kernel quadrature transform}; the corresponding quadrature rules are denoted by $\hat{I}_{1}^{\mathrm{OKQ}}, \dots, \hat{I}_{M}^{\mathrm{OKQ}}$, and the resulting approximation \eqref{def:hat_phi_x} by $\hat{f}_{\mathrm{OKQ},M,\bm{x}}$.

\rev{
To put things in perspective, note that $\hat{f}_{\mathrm{OKQ},M,\bm{x}}$ is a projection of $\hat{f}_{\mathrm{OKA},\bm{x}}$ introduced in \eqref{eq:interpolation_opt_problem}.
Indeed, for $m \in [M]$, using \eqref{e:OKA_weights} and \eqref{eq:OKQ_alpha_n}, 
\begin{equation}\label{eq:hat_I_m_identity_2}
\hat{I}_{m}^{\mathrm{OKQ}}(f) = f(\bm{x})^{\mathrm{T}}\bm{\alpha}_{m}^{\mathrm{OKQ}}  = \sigma_{m}\widehat{\bm{w}}_{\mathrm{OKA},\bm{x}}^{\mathrm{T}}e_{m}(\bm{x})
\end{equation} 
By the reproducing property,
\begin{equation}\label{eq:hat_I_m_identity_2}
  \hat{I}_{m}^{\mathrm{OKQ}}(f) = \Big\langle \sum\limits_{i \in [N]} (\widehat{\bm{w}}_{\mathrm{OKA},\bm{x}})_i k(x_i,.), e_m \Big\rangle_{\omega}= \langle \hat{f}_{\mathrm{OKA},\bm{x}}, e_m \rangle_{\omega},
\end{equation}
where the OKA weights are defined in \eqref{e:OKA_weights}.
Therefore
\begin{equation}\label{eq:representation_OKQ_OKA}
\hat{f}_{\mathrm{OKQ},M,\bm{x}} = \Pi_{\mathcal{E}_{M}}^{\|.\|_{\omega}} \hat{f}_{\mathrm{OKA},\bm{x}}.
\end{equation}
}

A second interpolative quadrature consists in taking $\kappa=k_{0,N}$ in \eqref{eq:interpolative_kernel_quadrature} where $k_{0,N}$ is the truncated kernel  defined in \eqref{e:different_truncated_kernels}.
Equivalently, we require that \rev{$M=N$ and}   
\begin{equation}\label{eq:pseudo_GQ}
\forall m,m' \in [N], \:\: \hat{I}_{m}(e_{m'}) = \delta_{m,m'},
\end{equation}
which is reminiscent of a property satisfied by Gaussian quadrature \citep{Gau04}. 
This yields new quadrature weights in \eqref{e:quadrature},
\begin{equation}
  \label{eq:alpha_QI}
  \bm{\alpha}_{m}^{\mathrm{QI}} = \rev{\bm{K}_{0,N}(\bm{x})}^{-1}e_{m}(\bm{x}), \quad m \in [N].
\end{equation} 
Note the similarity with \eqref{eq:OKQ_alpha_n}. \rev{Note that we can use \eqref{eq:alpha_QI} to extend the definition of $\hat{I}_{m}$ to $m \geq N+1$. However, unlike the OKQ-based transform, this extension is difficult to study theoretically. }

Finally, note that the resulting approximation, denoted by $\hat{f}_{\mathrm{QI},N,\bm{x}}$ has been called {\em quasi-interpolant} (QI) or \emph{hyperinterpolant} in the literature \citep{Slo95}, since
\begin{equation}
  \label{eq:def_fQI}
\forall f \in \mathcal{E}_N = \mathrm{Span}(e_1, \dots, e_N), \quad \hat{f}_{\mathrm{QI},N,\bm{x}} = f = f_{N}.
\end{equation}
In the rest of the article, we denote the corresponding quadrature rules by $\hat{I}^{\mathrm{QI}}_{1}, \dots, \hat{I}^{\mathrm{QI}}_{N}$. 
  Note that, for $m \in [N]$, $\hat{I}^{\mathrm{QI}}_{m}$ depends on $N$ and $\bm{x}$ that are dropped from the notation for simplicity.
  We now turn to the last approximation studied in this article.



\paragraph{The empirical least-squares approximation}
Let $\bm{x} \in \mathcal{X}^N$ and $q: \mathcal{X} \rightarrow \mathbb{R}_{+}^{*}$. 
Consider the so-called \emph{empirical} semi-norm $\|.\|_{q,\bm{x}}$ defined on $\mathbb{L}_{2}(\omega)$ by
\begin{equation}
  \label{def:ELS_seminorm}
  \|h\|_{q,\bm{x}}^2:= \frac{1}{N} \sum\limits_{i \in [N]} q(x_{i})h(x_{i})^2. 
\end{equation}
The empirical least-squares estimator yields yet another approximation\footnote{\rev{Note that in the literature, $\hat{f}_{\mathrm{ELS},M,\bm{x}}$ is often referred to as the weighted least-squares approximation, while $\hat{f}_{\mathrm{LS},\bm{x}}$ is called the $\mathbb{L}_2(\omega)$-projection.}}
\begin{equation}\label{def:hat_f_ELS_M}
\hat{f}_{\mathrm{ELS},M,\bm{x}}:= \rev{\Pi_{\mathcal{E}_M}^{\|.\|_{q,\bm{x}}} f} =  \argmin_{\hat{f} \in \mathcal{E}_{M}} \|f-\hat{f}\|_{q,\bm{x}}^2.
\end{equation}
We discuss the choice of $q$ and the design $\bm{x}$ in~\Cref{sec:iid_design_conf}.
The intuition is that, for well-chosen weight $q$ and design $\bm{x}$, the semi-norm $\|.\|_{q,\bm{x}}$ ``mimics" the $\|.\|_{\omega}$ norm as $N$ goes to infinity. 
Similarly, $\hat{f}_{\mathrm{ELS},M,\bm{x}}$ is supposed to inherit some of the properties of the projection $f_{M}$ of $f$ onto $\mathcal{E}_{M}$. 
Compared to $\hat{f}_{\mathrm{LS},\bm{x}}$, the approximation $\hat{f}_{\mathrm{ELS},M,\bm{x}}$ has the advantage of being computable given the evaluations of $f$ at the nodes $\bm{x}$. Indeed, \cite{CoDaLe13} shows that writing $\hat{f}_{\mathrm{ELS},M,\bm{x}} = \sum_{m \in [M]}\eta_{m} e_{m}$ yields
\begin{equation}
  \label{eq:ELS_matrix_form}
  \bm{G}_{q,\bm{x}}\bm{\eta} = \bm{d}_{q,\bm{x}},
\end{equation}
where $\bm{G}_{q,\bm{x}} = (\langle e_{m},e_{m'} \rangle_{q,\bm{x}})_{m,m' \in [M]} \in \mathbb{R}^{M \times M}$ is the Gram matrix of the family $(e_{m})_{m \in [M]}$ with respect to the inner product defined by~\eqref{def:ELS_seminorm}, and $\bm{d}_{q,\bm{x}}=(d_{m})_{m \in [M]} \in \mathbb{R}^{M}$ is defined by 
$$
  d_{m} := \sum_{i\in [N]}q(x_i)f(x_i)e_{m}(x_i)/N, \quad m \in [M].
$$
In other words, the numerical evaluation of $\hat{f}_{\mathrm{ELS},M,\bm{x}}$ requires to solve the linear system~\eqref{eq:ELS_matrix_form}. 
This in turn requires the evaluation of $\bm{G}_{q,\bm{x}}$ and $\bm{d}_{q,\bm{x}}$, which boils down to the evaluation of the functions $f$, $e_{m}$ and $q$ on the nodes $x_{1}, \dots,x_{N}$. 
Moreover, $\hat{f}_{\mathrm{ELS},M,\bm{x}}$ is uniquely defined if and only if $\bm{G}_{q,\bm{x}}$ is non-singular. 

\begin{table}
\centering
\caption{\rev{Summary of our bounds on the mean square error (MSE) of various approximations. The first two rows correspond to past work, where the error is measured in the RKHS norm. The other rows correspond to this paper, where the error is measured in $L^2$ norm. 
Note that results for CVS hold under the condition that $\beta_N$ in \eqref{e:def_betaN} is bounded, so that $\epsilon_1(N) = \mathcal{O}(\sigma_{N+1})$.}}
\label{tab:summary}
\begin{tabular}{|c|c|c|c|c|c|}
\hline
\textbf{Approx.} & \textbf{Target} & \textbf{Design} & \textbf{Norm} & \textbf{MSE} is $\mathcal{O}(\cdot)$ & \textbf{Reference}  \\ \hline

$\hat{f}_{\mathrm{OKA}, \bm{x}}$ & $f \in \bm{\Sigma}\mathbb{L}_{2}(\omega)$ & DPP &  $\|.\|_{\mathcal{F}}$ & $\sum_{m\geq N+1}\sigma_{m}$ & \cite{Bel21}, \cite{BeBaCh19} \\ \hline
$\hat{f}_{\mathrm{OKA}, \bm{x}}$ & $f \in \bm{\Sigma}\mathbb{L}_{2}(\omega)$  & CVS & $\|.\|_{\mathcal{F}}$ & $\epsilon_{1}(N)$ & \cite{BeBaCh20} \\ \hline
$\hat{f}_{\mathrm{LS}, \bm{x}}$ & $f \in \bm{\Sigma}\mathbb{L}_{2}(\omega)$  & DPP & $\|.\|_{\omega}$  & $\sum_{m\geq N+1}\sigma_{m}^2$ & \Cref{cor:dpprates}  \\ 
\hline
$\hat{f}_{\mathrm{OKA}, \bm{x}}$ &  $f \in \bm{\Sigma}\mathbb{L}_{2}(\omega)$ & CVS & $\|.\|_{\omega}$ & $\epsilon_{1}(N)$ & \Cref{prop:f_sigmag_cvs_omega}  \\ 
\hline
$\hat{f}_{\mathrm{LS}, \bm{x}}$ &  $f \in \bm{\Sigma}\mathbb{L}_{2}(\omega)$ & CVS & $\|.\|_{\omega}$ & $\epsilon_{1}(N)$ & \Cref{cor:cvsrates}  \\ 
\hline
$\hat{f}_{\mathrm{LS}, \bm{x}}$ & $f \in \mathcal{F}$  & DPP & $\|.\|_{\omega}$  & $\sum_{m\geq N+1}\sigma_{m}$ & \Cref{cor:dpprates}  \\ 
\hline
$\hat{f}_{\mathrm{OKA}, \bm{x}}$ &  $f \in \mathcal{F}$ & CVS & $\|.\|_{\omega}$ & $\inf\limits_{M \in \mathbb{N}^{*}}\Big(\sum_{m \in [M]} \epsilon_{m}(N) + \sigma_{M+1}\Big)$ & \Cref{thm:oka_bounds}  \\ 
\hline
$\hat{f}_{\mathrm{LS}, \bm{x}}$ &  $f \in \mathcal{F}$ & CVS & $\|.\|_{\omega}$ & $\inf\limits_{M \in \mathbb{N}^{*}}\Big(\sum_{m \in [M]} \epsilon_{m}(N) + \sigma_{M+1} \Big)$ & \Cref{cor:cvsrates}  \\ 
\hline
$\hat{f}_{\mathrm{tELS}, N,\bm{x}}$ & $f \in \mathcal{F}$ & DPP & $\|.\|_{\omega}$ & $\inf\limits_{M \in \mathbb{N}^{*}}\Big(M \sigma_{N+1} + \sigma_{M+1} \Big)$ & \Cref{prop:bound_tELS_DPP} \\ \hline
\end{tabular}
\end{table}


\rev{\Cref{tab:summary} illustrates the various approximations introduced so far.
It is a summary of the bounds in mean square error proven in previous work (in RKHS norm) and in this work (in $L^2$ norm), under two random designs, which we introduce in Section~\ref{sec:design_finite_dim}. \Cref{tab:summary_2} gathers these bounds for the special case $\sigma_m = C m^{-\alpha}$ where $C>0$ and $\alpha>1$.}

\begin{table}
\centering
\caption{\rev{Summary of our bounds on the mean square error (MSE) of various approximations under the assumption that the spectrum decreases as $N^{-\alpha}$, $\alpha>1$.}}
\label{tab:summary_2}
\begin{tabular}{|c|c|c|c|c|c|}
\hline
\textbf{Approximation} & \textbf{Target} & \textbf{Design} & \textbf{Norm} & \textbf{MSE} is $\mathcal{O}(\cdot)$ & \textbf{Reference}  \\ \hline

$\hat{f}_{\mathrm{OKA}, \bm{x}}$ & $f \in \bm{\Sigma}\mathbb{L}_{2}(\omega)$ & DPP &  $\|.\|_{\mathcal{F}}$ & $N^{1-\alpha} $ & \cite{Bel21}, \cite{BeBaCh19} \\ \hline
$\hat{f}_{\mathrm{OKA}, \bm{x}}$ & $f \in \bm{\Sigma}\mathbb{L}_{2}(\omega)$  & CVS & $\|.\|_{\mathcal{F}}$ & $N^{-\alpha} $ & \cite{BeBaCh20} \\ \hline
$\hat{f}_{\mathrm{LS}, \bm{x}}$ & $f \in \bm{\Sigma}\mathbb{L}_{2}(\omega)$  & DPP & $\|.\|_{\omega}$  & $N^{1-2\alpha} $ & \Cref{cor:dpprates}  \\ 
\hline
$\hat{f}_{\mathrm{OKA}, \bm{x}}$ &  $f \in \bm{\Sigma}\mathbb{L}_{2}(\omega)$ & CVS & $\|.\|_{\omega}$ & $N^{-\alpha} $ & \Cref{prop:f_sigmag_cvs_omega}  \\ 
\hline
$\hat{f}_{\mathrm{LS}, \bm{x}}$ &  $f \in \bm{\Sigma}\mathbb{L}_{2}(\omega)$ & CVS & $\|.\|_{\omega}$ & $N^{-\alpha} $  & \Cref{cor:cvsrates}  \\ 
\hline
$\hat{f}_{\mathrm{LS}, \bm{x}}$ & $f \in \mathcal{F}$  & DPP & $\|.\|_{\omega}$  & $N^{1-\alpha} $ & \Cref{cor:dpprates}  \\ 
\hline
$\hat{f}_{\mathrm{OKA}, \bm{x}}$ &  $f \in \mathcal{F}$ & CVS & $\|.\|_{\omega}$ & $N^{-\alpha\frac{\alpha}{\alpha+1}} $ & \Cref{cor:oka_polynomial_rate}  \\ 
\hline
$\hat{f}_{\mathrm{OKA}, \bm{x}}$ &  $f \in \mathcal{F}$ & DPP & $\|.\|_{\omega}$ & $ N^{-\alpha\frac{\alpha-1}{\alpha+1}}$ & \Cref{cor:oka_polynomial_rate}  \\ 
\hline
$\hat{f}_{\mathrm{tELS}, N,\bm{x}}$ & $f \in \mathcal{F}$ & DPP & $\|.\|_{\omega}$ & $N^{-\alpha\frac{\alpha}{\alpha+1}} $ & \Cref{cor:tels_polynomial_rate}  \\ \hline
\end{tabular}
\end{table}






\subsection{Designs for finite-dimensional approximations}\label{sec:design_finite_dim}

An abundant literature provides theoretical guarantees for the finite-dimensional approximations presented in~\Cref{sec:non-asymptotic}. This includes, but is not limited to, \citep{ErTu37,Slo95,Wend04,ScWe06}. 
These results deal with specific RKHSs, and cannot be easily generalized to an arbitrary RKHS.
Recently, a new tendency has emerged that looks for a universal sampling approach that is valid for a large class of RKHSs. 
This section first provides an overview of this universal sampling literature.
In the same spirit, we then show that determinantal point processes (DPPs) offer an adequate framework to design configurations of nodes with strong theoretical guarantees for finite-dimensional approximations, in a wide family of RKHSs. 

\subsubsection{Independent samples from the Christoffel function}
\label{sec:iid_design_conf}

The last decade has seen significant progress in the study of function reconstruction based on randomized configurations \citep{CoDaLe13,HaDo15,CoMi17,AdCa20,AdCaDeMo22,DoCo22}.
These works focused on the study of the empirical least-squares approximation; see~\Cref{sec:eigenspace_approximations}. 
In particular, the so-called \emph{instance optimality property} (IOP) has been a matter of extensive investigation.
The idea is to find assumptions on the configuration $\bm{x}$, the function $q: \mathcal{X} \rightarrow \mathbb{R}_{+}^{*}$, and the order $M \leq N$, under which we can certify that\footnote{This is a specific formulation of the IOP. See \cite{CoMi17} for a more generic formulation.}
 \begin{equation}\label{eq:IOP}
\forall f \in \mathcal{F}, \:\: \|f -\hat{f}_{\mathrm{ELS},M,\bm{x}}\|_{\omega} \leq C \|f-f_{M}\|_{\omega},
 \end{equation}
 where $C>0$ is a constant and \rev{$\hat{f}_{\mathrm{ELS},M,\bm{x}}$ is defined in \eqref{def:hat_f_ELS_M}.}
 Instance optimality implies that $\hat{f}_{\mathrm{ELS},M,\bm{x}} = f$ for $f \in \mathcal{E}_{M}$, so that the reconstruction is exact in the eigenspace $\mathcal{E}_{M}$. 
 Moreover, it provides an upper bound of the approximation error $\|f -\hat{f}_{\mathrm{ELS},M,\bm{x}}\|_{\omega}$ for generic functions. 
 Investigating necessary conditions for the IOP unraveled the importance of the study of the eigenvalues of the Gram matrix $\bm{G}_{q,\bm{x}}$ \rev{defined after \eqref{eq:ELS_matrix_form}}; see \citep{CoMi17}. 
 In particular, it was proved that 
 \begin{equation}
  \label{eq:rip}
  \|\bm{G}_{q,\bm{x}} - \mathbb{I}_{M}\|_{\mathrm{op}} \leq \delta \iff \forall f \in \mathcal{E}_{M}, \:\: (1-\delta)\|f\|_{\omega}^2 \leq \|f\|_{q,\bm{x}}^2 \leq (1+\delta)\|f\|_{\omega}^2,
\end{equation}
 where $\mathbb{I}_{M} \in \mathbb{R}^{M \times M}$ is the identity matrix of order $M$, and $\|.\|_{\mathrm{op}}$ is the operator norm. 
 In other words, the $\|.\|_{q,\bm{x}}$ norm is equivalent to $\|.\|_{\omega}$ on $\mathcal{E}_{M}$ if and only if the matrix $\bm{G}_{q,\bm{x}}$ is close to $\mathbb{I}_{M}$. 
 Interestingly, \cite{Gro20} showed that the condition \eqref{eq:rip} was connected to Marcinkiewicz-Zygmund inequalities, which are common tools in the literature of non-uniform sampling \citep{Gro93,OrSa07,FiMh11}; see \citep{Gro20} for more references.
\rev{\cite{CoMi17,HaNoPe22}} proved that drawing $x_{1}, \dots, x_{N}$ proportionally to
\begin{equation}
  \label{e:not_really_iid_christoffel_sampling}
  1_{\|\bm{G}_{q,\bm{x}}-\mathbb{I}_{M}\|_{\mathrm{op}} \leq 1/2} \times \prod_{i \in [N]} \frac{1}{q(x_i)}\mathrm{d}{\omega}(x_i),
\end{equation}
the IOP holds with large probability provided that $N$ scales as $\|qc_{M}\|_{\infty}\log(2M)$, where $c_{M}$ is the inverse of the so-called \emph{Christoffel function}
\begin{equation}\label{eq:christoffel_expression}
c_{M}(x): = \sum\limits_{m \in [M]}e_{m}(x)^2; \:\: x \in \mathcal{X}.
\end{equation}
Since the required sampling budget \rev{$N$} grows with $\|qc_{M}\|_{\infty}$,
it is preferable to choose the function $q$ in such a way that $\|qc_{M}\|_{\infty}$ is minimized, under the constraint that $\int_{\mathcal{X}}1/q(x) \mathrm{d}\omega(x) = 1$. 
In particular, if we take $q = M/c_{M}$, the constant $\|qc_{M}\|_{\infty}$ grows linearly with $M$, and the required sampling budget \rev{$N$} scales as $M\log(M)$.
This is to be compared with the situation when $q$ is taken to be a constant, where the required sampling budget \rev{$N$} would scale as \rev{$\mathcal{O}(\|c_{M}\|_{\infty} \log(M))$}. The constant $\|c_{M}\|_{\infty}$, also called the \emph{Nikolskii constant} in approximation theory, is known to be at best \rev{superlinear in $M$ (of order $M^{\alpha}$ for $\alpha>1$)} for many families of orthogonal polynomials \citep{Nev86,Bos94,Xu96,Tot00}, hinting at a suboptimal sampling budget in \rev{$\mathcal{O}(M^\alpha \log(M))$}. 
Note also that \eqref{e:not_really_iid_christoffel_sampling} with $q$ inversely proportional to $c_M$ can be sampled by rejection sampling using i.i.d. draws from $1/q$, with controllable rejection probability \citep{CoMi17}.
In spite of the rejection step to account for the constraint in \eqref{e:not_really_iid_christoffel_sampling}, we still abusively denote the method as i.i.d. Christoffel sampling.

Achieving the IOP using a sampling budget \rev{$N$} that scales linearly in $M$ was proved to be possible using sparsification techniques \citep{DoCo22c,ChDo23}. 
However, as was shown in \citep{DoCo22c}, \rev{the constant $\delta$ worsens as $N$ approaches $M$, and the bound does not hold when $M = N$.}
Another line of research has focused on obtaining randomized approximations with optimal worst-case mean squared $L^2$ error, where \emph{worst-case} is meant over the unit ball of the RKHS. 
With a budget of $N$ evaluations, the latter error was shown to be at least $\sigma_{2N}/\sqrt{2}$ by \cite{Nov92}.
Constructive solutions nearly achieving \citep{WasWoz07} or achieving \citep{Kri19} the optimal rate followed, under conditions on the sequence $(\sigma_m)$.
In particular, the approximation of \cite{Kri19} is based on a multi-level Monte Carlo method that uses i.i.d. Christoffel sampling in its proposal mechanism. 
One disadvantage of the approaches of \citep{WasWoz07,Kri19} is the strong assumptions on the decay of the spectrum $(\sigma_m)$, which prevents, for instance, the spectrum to decay exponentially.

A related approach was adopted by \cite{Bac17}, yet with a kernel-based approximation that solves a regularized variant of the optimization problem~\eqref{eq:interpolation_opt_problem}. 
In particular, their analysis relied on the so-called \emph{regularized leverage score} function
\begin{equation}\label{eq:reg_lvs}
\tilde{c}_{\lambda}(x):= \sum\limits_{m \in \mathbb{N}^{*}}\frac{\sigma_m}{\sigma_m +\lambda}e_{m}(x)^2,
\end{equation}
where $\lambda>0$ is a regularization constant. 
\cite{Bac17} showed that when the nodes are i.i.d. draws from $\tilde{c}_{\lambda} \mathrm{d}\omega$, the resulting approximation converges to $0$ at an almost optimal rate, as $\lambda$ is made to go to zero with $N$ at a suitable rate. 
One practical downside, compared to \eqref{eq:christoffel_expression}, is the need for an infinite summation, \rev{and some approximation is needed, as was proposed in \cite{ChScDeRo23}.}

\subsubsection{Determinantal sampling}
\label{s:DPPs}
In \citep{BeBaCh19} and \citep{BeBaCh20}, we have investigated two random designs defined using determinants and Gram matrices. 

\begin{definition}[\rev{A family of projection DPPs}]
  \label{def:dpp}
\rev{Let $N\geq 1$, and $T \subset \mathbb{N}^{*}$ such that $|T| = N$. 
Consider the kernel $k_{0,T}$ of the projection onto the span of the eigenfunctions of $\bm{\Sigma}$ indexed by $T$, as defined in \cref{s:different_kernels}.
The design $\bm{x} = (x_1, \dots, x_N)$ is said to have for distribution the DPP of kernel $k_{0,T}$ and reference measure $\omega$ if
  \begin{equation}\label{def:density_detsampling_T}
    \bm{x} = (x_1, \dots, x_N) \sim \frac{1}{N!} 
      \Det \Big(k_{0,T}(x_{i},x_{j})\Big)_{(i,j) \in [N]\times [N]} \mathrm{d}\omega(x_1) \dots \mathrm{d}\omega(x_N).
    \end{equation}}
\end{definition}


\rev{
  First, the fact that \eqref{def:density_detsampling_T} integrates to $1$ is a consequence of $\rev{k_{0,T}}$ being a projection kernel.
  To insist on the latter property, the set $\{x_1, \dots, x_n\}$ is thus called the \emph{projection} DPP of kernel $k_{0,T}$ in the literature \citep{HoKrPeVi06}.
  Since order does not play a role in our paper, by an abuse of language, we also call the distribution of the \emph{vector} $\bm{x}$ a projection DPP.
  Second, in the remainder of the paper, we shall often consider \eqref{def:density_detsampling_T} with $T=\{1, \dots, N\}$, in which case we write an expectation under \eqref{def:density_detsampling_T} as $\mathbb{E}_{\mathrm{DPP}}$.
}

DPPs were introduced by \cite{Mac75}, and possess many interesting properties \citep{HoKrPeVi06}. For instance, \rev{the projection DPP studied in \citep{BeBaCh19} corresponds to the case $T=\{1,\dots,N\}$, and} any point of $\bm{x}$ in \eqref{def:density_detsampling_T} has marginal distribution 
\begin{equation}
  \label{e:marginal}
  x_i\sim \frac{1}{N} \sum_{n \in [N]} e_n(x_i)^2 \mathrm d\omega(x_i),
\end{equation}
which is related to the inverse of the Christoffel function \eqref{eq:christoffel_expression}.
In that sense, \eqref{def:density_detsampling_T} \rev{extends previous
work on designs} sampled from the inverse of the Christoffel function, \rev{by} adding a kernel-dependent correlation among the nodes.
Another useful property of DPPs with projection kernels like \eqref{def:density_detsampling_T} is that the chain rule for $\bm{x}$ is explicit.
This yields a polynomial-time, \emph{exact} sampling algorithm colloquially known as HKPV, after the authors of \citep{HoKrPeVi06}.


\emph{Continuous volume sampling}, introduced in \citep{BeBaCh20}, is a related distribution for nodes, which relies on the Gram matrix of the RKHS kernel instead of a projection kernel.

\begin{definition}[continuous volume sampling]
  \label{def:vs}

  The design $\bm{x} = (x_1, \dots, x_N)$ is said to have a distribution according to continuous volume sampling if 
  \begin{equation}\label{def:density_vs}
    \bm{x} = (x_1, \dots, x_N) \sim \frac{1}{Z} 
      \Det \Big(k(x_{i},x_{j})\Big)_{(i,j) \in [N]\times [N]} \mathrm{d}\omega(x_1) \dots \mathrm{d}\omega(x_N),
    \end{equation}
  where $Z<\infty$ is a normalization constant.
\end{definition}
In the remainder of the paper, $\mathbb{E}_{\mathrm{CVS}}$ is to be understood as an expectation under \eqref{def:density_vs}.
Note that unlike \eqref{def:density_detsampling_T}, the normalization constant is not explicit. 
Yet Hadamard's inequality yields
\begin{align*}
  Z & = \int_{\mathcal{X}^{N}} \Det \Big(k(x_{i},x_{j})\Big) \prod_{n \in [N]} \mathrm{d}\omega(x_{n}) \leq \int_{\mathcal{X}^{N}} \prod_{n \in [N]} k(x_{n},x_{n}) \mathrm{d}\omega(x_{n})\\
  &\leq \Big(\int_{\mathcal{X}}  k(x,x) \mathrm{d}\omega(x)\Big)^{N} <+\infty.
\end{align*}

While the chain rule for continuous volume sampling is not as simple as for the DPP in \eqref{def:density_detsampling_T}, continuous volume sampling can actually be shown to be a statistical mixture of projection DPPs.
\rev{
  Formally, for a bounded test function $h$,
  \begin{equation}
    \label{e:cvs_as_mixture}
    \mathbb{E}_{\mathrm{CVS}} h(\bm{x}) = \sum_{T \subset \mathbb{N}^*, \:\: |T|=N} \pi_T \int h(\bm{x}) \Det \Big(k_{0,T}(x_{i},x_{j})\Big)_{(i,j) \in [N]\times [N]} \mathrm{d}\omega(x_1) \dots \mathrm{d}\omega(x_N),
  \end{equation}
  where 
  $$
    \pi_T = \frac{\prod_{t\in T}\sigma_{t}}{\sum_{I\subset \mathbb{N}^*, \:\: |I|=N}\prod_{i\in I}\sigma_{i}},
  $$
  and $k_{0,T}$ is defined by \eqref{e:different_truncated_kernels}.} In words, $\bm{x}$ in \eqref{def:density_vs} can be drawn by first sampling a subset $I \subset\mathbb{N}^*, |I|=N$ proportionally to $\prod_{i\in I}\sigma_{i}$, and then drawing $\bm x$ from \eqref{def:density_detsampling_T} with $T= I$.
In that sense, continuous volume sampling is a ``soft" modification of the DPP \eqref{def:density_detsampling_T} \rev{corresponding to $T=\{1, \dots, N\}$}, which is the component with largest weight in the mixture; see \citep{BeBaCh20} for more details.

\subsection{Existing results on RKHS sampling using DPPs and CVS}
\label{s:DPPs_th_old}

This section gathers existing work on function reconstruction using determinantal distributions.
We provide some details about some results that will be useful in~\Cref{sec:new_results}.

\subsubsection{The Ermakov-Zolotukhin quadrature rule}\label{sec:review_ezq}
 
Non-asymptotic reconstruction guarantees for functions living in Hilbert spaces can be traced back to the work of \cite{ErZo60}. 
The authors studied quadrature rules $\hat{I}_{1}, \dots, \hat{I}_{N}$ obtained by taking nodes $\bm{x}$ following the projection DPP of~\Cref{def:dpp}, and the corresponding vectors of weights $\bm{\alpha}_{1}, \dots, \bm{\alpha}_{N}$ \rev{as in}~\eqref{eq:alpha_QI}. 
In particular, they proved that for a continuous function $f:\mathcal{X} \rightarrow \mathbb{R}$ living in $\mathbb{L}_{2}(\omega)$, 
\begin{equation}\label{eq:ez_1}
 \forall n \in [N], \:\:\left\{
    \begin{array}{ll}
        \mathbb{E}_{\DPP}\hat{I}_{n}(f)  & = I_{n}(f) \\
        \mathbb{V}_{\DPP}\hat{I}_{n}(f)  & = \sum_{m \geq N+1}\langle f,e_{m} \rangle_{\omega}^2 \\
    \end{array}
\right. .
\end{equation}
\cite{GaBaVa19} revisited this result and proved that
\begin{equation}\label{eq:ez_2}
 \forall n_{1},n_{2} \in [N], \:\: n_{1} \neq n_{2} \implies \mathbb{C}\mathrm{ov}_{\DPP} (\hat{I}_{n_1}(f),\hat{I}_{n_2}(f)) = 0.
\end{equation}
As mentioned in \citep{KaAd22}, \eqref{eq:ez_1} implies that the resulting transform $\hat{f}_{\mathrm{QI},N,\bm{x}}$ introduced in \Cref{sec:non-asymptotic} satisfies
\begin{equation}\label{eq:KaLe_result}
\forall f \in \mathbb{L}_{2}(\omega), \:\: \mathbb{E}_{\DPP}\|f-\hat{f}_{\mathrm{QI},N,\bm{x}}\|_{\omega}^2 = N \|f-f_{N}\|_{\omega}^2.
\end{equation}
Thus the quasi-interpolant satisfies an instance optimality property. Yet, the corresponding constant grows to infinity with $N$. Observe that~\eqref{eq:KaLe_result} holds in $\mathbb{L}_{2}(\omega)$ and $f$ is not assumed to live in a particular RKHS.

\subsubsection{The optimal kernel approximation using determinantal sampling}\label{sec:review_qdpp}
The study of quadrature rules instigated the inquiry of  non-asymptotic guarantees for finite-dimensional approximations based on determinantal nodes for functions living in RKHSs. 
Indeed, in the context of numerical integration, for any function $h$ living in the RKHS $\mathcal{F}$ and for any $g \in \mathbb{L}_{2}(\omega)$, letting $f=\bm{\Sigma}g$ \citep{MuFuSrSc17},
\begin{equation}\label{eq:quadrature_error_upper_bound}
\Big| \int_{\mathcal{X}} h(x)g(x) \mathrm{d}\omega(x) - \sum\limits_{i \in [N]} w_{i}h(x_{i})\Big|^2 \leq \|h\|_{\mathcal{F}}^2 \|f-\sum\limits_{i \in [N]} w_{i}k(x_i,\cdot)\|_{\mathcal{F}}^2,
\end{equation}
%
Moreover, \rev{equality is attained for some $h \in \mathcal{F}$.}
In other words, \rev{if $\|h\|_{\mathcal{F}} \leq 1$,} the squared residual provides an upper bound of the squared error of the approximation of the integral $\int_{\mathcal{X}} h(x) g(x) \mathrm{d}\omega(x)$ by the quadrature rule $\sum_{i \in [N]} w_{i}h(x_i)$. This is especially applicable to the optimal kernel quadrature \rev{used for computing the coefficients $\bm{\alpha}_{m}^{\mathrm{OKQ}}$} in~\Cref{sec:eigenspace_approximations}: taking the vector of weights \rev{$(w_i)_{i \in [N]}$} equal to $\bm{K}_{1}(\bm{x})^{-1} f(\bm{x})= \bm{K}_{1}(\bm{x})^{-1} \bm{\Sigma}g(\bm{x})$ extends~\eqref{eq:OKQ_alpha_n} to any $g$. The squared worst-case integration error of the corresponding quadrature rule results equal to $\|f-\hat{f}_{\mathrm{OKA},\bm{x}}\|_{\mathcal{F}}^2$. Now,
it was proven in \citep{BeBaCh19} that
\begin{equation}\label{eq:main_result_DPP_sup}
\EX_{\DPP} \sup_{f \in \bm{\Sigma}\mathbb{B}_{\mathbb{L}_{2}(\omega)}} \|f - \hat{f}_{\mathrm{OKA},\bm{x}}\|_{\mathcal{F}}^{2} = \mathcal{O}(N^{2} r_{N+1}),
\end{equation}
where 
\begin{equation}
    \label{eq:def_r_Nplus1}
    r_{N+1}:= \sum_{m \geq N+1} \sigma_{m}.
\end{equation} 
Moreover, it was proven in Theorem 3 of \citep{Bel21} that
\begin{equation}\label{eq:result_ezq_2}
 \EX_{\DPP} \|f - \hat{f}_{\mathrm{OKA}, \bm{x}} \|_{\mathcal{F}}^{2}  \leq 4 \|g\|_{\omega}^{2} r_{N+1}.
\end{equation}
The first upper bound~\eqref{eq:main_result_DPP_sup} deals with the worst interpolation error on the set $\bm{\Sigma} \mathbb{B}_{\mathbb{L}_2(\omega)}$.
In contrast, the second upper bound~\eqref{eq:result_ezq_2} is not uniform in $f$. 
These upper bounds highlight the importance of the eigenvalues $\sigma_{m}$ for the study of the convergence of $\hat{f}_{\mathrm{OKA}, \bm{x}}$ under the projection DPP of~\Cref{def:dpp}:  
the upper bounds converge to $0$ only if the convergence to zero of the tail $r_N$ of the eigenvalues is fast enough. 

These results give convergence rates for the interpolation under the distribution of the projection DPP, that scale, at best, as $\mathcal{O}(r_N)$, which is slower than the empirical convergence rate $\mathcal{O}(\sigma_{N+1})$ observed in \citep{Bel20, BeBaCh19}. 
Indeed, when $\sigma_{N} = N^{-2s}$ for $s>1/2$, we have $\mathcal{O}(r_{N}) = \mathcal{O}(N^{1-2s})$, which is slower, by a factor of $N$, than $\mathcal{O}(\sigma_{N+1})$. Recall that $\mathcal{O}(\sigma_{N+1})$ corresponds to the optimal rate of convergence in the following sense: 
For $N \in \Ns$ and $\bm{x} \in \mathcal{X}^N$ such that $\Det \bm{\kappa}(\bm{x}) >0$, there exists $g \in \mathbb{L}_{2}(\omega)$ such that $\|g\|_{\omega} \leq 1$ and $\|f - \hat{f}_{\mathrm{OKA}, \bm{x}}\|_{\F}^{2} \geq \sigma_{N+1}$;
 see Section 2.5 in \cite{BeBaCh20} for a proof. 

On the other hand, continuous volume sampling (CVS) permits to derive better convergence guarantees. Indeed, \cite{BeBaCh20} showed that
\begin{equation}\label{CVS_eq:main_result_EX_VS_err_mu}
\EX_{\VS} \|f - \hat{f}_{\mathrm{OKA}, \bm{x}}\|_{\F}^{2} = \sum\limits_{m \in \mathbb{N}^{*}} \langle g, e_{m}\rangle_{\omega}^{2} \epsilon_{m}(N),
\end{equation}
where 
\begin{equation}\label{eq:def_epsilon}
\epsilon_{m}(N) = \sigma_{m} \frac{ \sum\limits_{  U \in \: \UN; m\notin U} \prod\limits_{u \in U} \sigma_{u}}{\sum\limits_{ U \in \: \UN} \prod\limits_{u \in U} \sigma_{u} } ; \:\:\:\:\mathcal{U}_{N}:=\{ U \subset \mathbb{N}^{*}; |U|=N \}.
\end{equation}
Identity~\eqref{CVS_eq:main_result_EX_VS_err_mu} gives an explicit expression of $\EX_{\VS} \|f - \hat{f}_{\mathrm{OKA}, \bm{x}}\|_{\F}^{2}$ in terms of the coefficients of $g$ on the o.n.b. $(e_m)_{m \in \mathbb{N}^*}$ and the $\epsilon_{m}(N)$ defined by~\eqref{eq:def_epsilon}. Moreover, by observing that the sequence $(\epsilon_m(N))_{m \in \Ns}$ is non-increasing, the formula~\eqref{CVS_eq:main_result_EX_VS_err_mu} implies that
\begin{equation}\label{CVS_eq:upper_bound_sup_epsilon}
\sup\limits_{f \in \bm{\Sigma}\mathbb{B}_{\mathbb{L}_{2}(\omega)}} \EX_{\VS} \|f - \hat{f}_{\mathrm{OKA}, \bm{x}}\|_{\F}^{2} \leq \sup\limits_{m \in \mathbb{N}^{*}} \epsilon_{m}(N) = \epsilon_1(N).
\end{equation}
As stated by Theorem 4 of \citep{BeBaCh20}, one has
\begin{equation}\label{CVS_eq:ineq_r_N}
\epsilon_{1}(N) \leq \sigma_{N+1} \left(1+ \beta_{N}\right),
\end{equation}
where 
\begin{equation}
  \label{e:def_betaN}
  \beta_{N} := \min_{M \in [2:N+1]} \left[(N-M+2)\sigma_{N+1}\right]^{-1} \sum_{m \geq M} \sigma_m
\end{equation} 
In particular, 
under the assumption that the sequence $(\beta_N)_{N \in \mathbb{N}^*}$ is bounded, which is the case as soon as the sequence $(\sigma_{m})$ decreases polynomially or exponentially, \eqref{CVS_eq:ineq_r_N} yields 
\begin{equation}\label{eq:cvs_commenting_bound}
\sup_{f \in \bm{\Sigma}\mathbb{B}_{\mathbb{L}_{2}(\omega)}}\EX_{\VS} \|f - \hat{f}_{\mathrm{OKA}, \bm{x}}\|_{\F}^{2} = \mathcal{O}(\sigma_{N+1}),
\end{equation}
which corresponds to the optimal rate of convergence.

The results reviewed until now are restricted to smooth functions that belong to $\bm{\Sigma}\mathbb{L}_{2}(\omega)$. They are mostly relevant in the study of kernel-based quadrature using determinantal sampling.
Yet, as it was mentioned , $\bm{\Sigma}\mathbb{L}_{2}(\omega)$ is strictly included in the RKHS $\mathcal{F}$. Still, it is possible to extend~\eqref{eq:cvs_commenting_bound} to functions belonging to $ \bm{\Sigma}^{r+1/2}\mathbb{L}_{2}(\omega)$,  where $r \in [0,1/2]$ is a parameter that interpolates between the set of the embeddings $\bm{\Sigma}\mathbb{L}_{2}(\omega)$ and the RKHS $\F = \bm{\Sigma}^{1/2}\mathbb{L}_{2}(\omega)$. Indeed, \cite{BeBaCh20} showed that $\EX_{\VS} \|f - \hat{f}_{\mathrm{OKA},\bm{x}}\|_{\F}^{2} = \mathcal{O}(\sigma_{N+1}^{2r})$ under the assumption that the sequence $(\beta_{N})_{N \in \mathbb{N}^{*}}$ is bounded.
This result is an extension of \eqref{eq:cvs_commenting_bound} to $r \leq 1/2$: 
 the rate of convergence is $\mathcal{O}(\sigma_{N+1}^{2r})$ which is slower than $\mathcal{O}(\sigma_{N+1})$ and gets worse as $r$ goes to $0$. In other words, an additional level of smoothness controlled by $r>0$ is needed to achieve the convergence with respect to the RKHS norm $\|.\|_{\F}$. Nevertheless, \rev{one can expect the convergence to hold with respect to the weaker norm} $\|.\|_{\omega}$, as studied in \Cref{sec:new_results}.

\section{Theoretical guarantees}\label{sec:new_results}
While previous work on determinantal sampling measured performance in RKHS norm, this work investigates the mean-square reconstruction error in $\Ltwo$ norm for the main approximations defined in \cref{sec:non-asymptotic}.
\rev{
We first derive bounds for the least-squares approximation under the projection DPP \eqref{def:dpp}, and then derive bounds for the optimal kernel approximation under both the DPP and volume sampling. 
As a corollary, we obtain bounds for the least-squares approximation under volume sampling.
We conclude with results on instance optimality for the truncated empirical least squares approximation under the projection DPP. 
Our results are summarized in \Cref{tab:summary} and \Cref{tab:summary_2}.
}

We emphasize that a few key results play a fundamental role, like the two properties  \eqref{eq:ez_1} and \eqref{eq:ez_2} of the Ermakov-Zolotukhin quadrature, and the `Pythagorean' formula \eqref{CVS_eq:main_result_EX_VS_err_mu} of continuous volume sampling.
\subsection{The least-squares approximation under the projection DPP}\label{sec:L2_optimal_approximation}

Our first result gives a bound for the mean-square error of $\hat{f}_{\mathrm{LS},\bm{x}}$ under the projection DPP of~\Cref{def:dpp}.

\begin{theorem}\label{thm:useful_result_DPP}
  Let $f \in \F$, and for $M\in \mathbb{N}^{*}$, let \rev{$f_{M}$} be its projection onto the eigenspace \rev{$\mathcal{E}_{M}$} defined by~\eqref{eq:eigen_approximation}. 
  Let $N\in \mathbb{N}^{*}$ and $\bm{x}=(x_1, \dots, x_N)$ be distributed according to the DPP in \eqref{def:density_detsampling_T} of kernel $k_{0,T}$ with $T=\{1, \dots, N\}$. 
  Then
  \begin{equation}
    \label{eq:main_result_EX_DPP_err_mu}
    \mathbb{E}_{\DPP}\|f -\hat{f}_{\mathrm{LS},\bm{x}} \|_{\omega}^{2} \leq  2 \min_{M\leq N} \Big(\|f-\rev{f_{M}}\|_{\omega}^2 + \rev{\sum\limits_{n \in [M]} \frac{\langle f,e_n \rangle_{\omega}^2}{\sigma_n ^2}} \sum\limits_{m \geq N+1} \sigma_{m}^2 \Big).
  \end{equation}
\end{theorem}
The proof is in \cref{proof:thm_useful_result_DPP}. 
The general bound \eqref{eq:main_result_EX_DPP_err_mu} is composite, but explicit rates of convergence can be derived under additional smoothness assumptions on $f$.
At one extreme, assuming $f = f_{M_0}$ for some $M_0\leq N$, the minimum in \eqref{eq:main_result_EX_DPP_err_mu} is reached for  $M=N$, which yields an expected squared error in $\mathcal{O}(\sum_{m \geq N+1} \sigma_m^2)$. 
A refined view of how smoothness impacts the rate of convergence can be obtained for the nested spaces in \Cref{sec:smoothness_frac}, as stated by \cref{cor:dpprates}.

\rev{
\begin{corollary}
  \label{cor:dpprates}
  Let $r \geq 0$ and $f \in \bm{\Sigma}^{1/2+r}\mathbb{L}_{2}(\omega)$. 
  Then 
  \begin{equation}
    \label{e:rnonnegativeDPP}
    \mathbb{E}_{\DPP}\|f -\hat{f}_{\mathrm{LS},\bm{x}} \|_{\omega}^{2} = \mathcal{O}\left(\sum_{m\geq N+1}\sigma_{m}\right).
  \end{equation}
  If $r>0$, one further has
  \begin{equation}
    \label{e:rpositiveDPP}
    \mathbb{E}_{\DPP}\|f -\hat{f}_{\mathrm{LS},\bm{x}} \|_{\omega}^{2} = o\left(\sum_{m\geq N+1}\sigma_{m}\right).
  \end{equation}
  Finally, if $r\geq 1/2$, 
  \begin{equation}
    \label{e:rlargerthanonehalfDPP}
    \mathbb{E}_{\DPP}\|f -\hat{f}_{\mathrm{LS},\bm{x}} \|_{\omega}^{2} = \mathcal{O}\left(\sum_{m\geq N+1}\sigma_{m}^2\right).
  \end{equation}
\end{corollary}
}

\rev{
A few comments are in order. 
First, such changes in the convergence rate when $r$ increases are called \emph{superconvergence} in the literature, see \citep{Sch18} and references therein.
Second, our numerical experiments in \Cref{s:numsims} suggest that the rates in \cref{cor:dpprates} might be improved further, although likely at the price of a significantly more complicated proof. 
Finally, it is possible to derive a counterpart to \Cref{cor:dpprates} for CVS, using \eqref{e:cvs_as_mixture} and applying \Cref{thm:useful_result_DPP} to each projection DPP of the mixture. 
However, we prefer to refer the reader to \Cref{sec:ls_cvs}, where a similar bound for CVS is obtained by using a simpler alternative proof.
}

\rev{
\begin{proof}[Proof of \Cref{cor:dpprates}]
  Let $r\geq 0$.
  Since $f \in \bm{\Sigma}^{1/2+r}\mathbb{L}_{2}(\omega)$, there exists $g \in \mathbb{L}_{2}(\omega)$ such that 
\begin{equation}
  \forall m \in \mathbb{N}^{*}, \:\: \langle f,e_{m} \rangle_{\omega} = \sigma_{m}^{r+1/2} \langle g,e_{m} \rangle_{\omega}.
\end{equation}
For a fixed $M$, the first term in the minimum of the r.h.s. of \eqref{eq:main_result_EX_DPP_err_mu} satisfies
\begin{align}
  \label{e:firstterm}
 \rev{ \|f-f_{M}\|_{\omega}^2} & = \rev{\sum\limits_{m \geq M+1}\langle f, e_{m} \rangle_{\omega}^2  = \sum\limits_{m \geq M+1}\frac{\langle f, e_{m} \rangle_{\omega}^2}{\sigma_{m}^{2r+1}}\sigma_{m}^{2r+1}} \\
  &  \rev{ \leq \sigma_{M+1}^{2r+1} \sum\limits_{m \geq M+1}\langle g, e_{m} \rangle_{\omega}^2 \leq \|g\|_{\omega}^2 \sigma_{M+1}^{2r+1} }.
\end{align}
\rev{Meanwhile, the second term in the r.h.s. of \eqref{eq:main_result_EX_DPP_err_mu} is bounded as follows,}
\begin{equation}
  \label{e:secondterm}
 \rev{\sum\limits_{n \in [M]} \frac{\langle f,e_n \rangle_{\omega}^2}{\sigma_n ^2} \sum\limits_{m \geq N+1} \sigma_{m}^2 \leq \|g\|_{\omega}^2 \max\big(\sigma_{M+1}^{2r-1}, \sigma_{1}^{2r-1}\big)   \sum\limits_{m \geq N+1} \sigma_{m}^2.}
\end{equation}
Taking $M=N$ in \eqref{e:firstterm} and \eqref{e:secondterm} leads to \eqref{e:rnonnegativeDPP}, \eqref{e:rpositiveDPP}, and \eqref{e:rlargerthanonehalfDPP}.
\end{proof}
}

\rev{
  An alternative illustration of how smoothness impacts the bound in \cref{thm:useful_result_DPP} comes from assuming a polynomial decay of the eigenvalues of the integration operator. 
\begin{corollary}
  \label{cor:dpprates_polynomialrate}
  With the notations of \cref{thm:useful_result_DPP}, assume that there is $s>1/2$ such that $\sigma_{m} = m^{-2s}$, $m\geq 1$. Then
  \begin{equation}
    \label{e:slargerthanonehalf}
    \mathbb{E}_{\DPP}\|f -\hat{f}_{\mathrm{LS},\bm{x}} \|_{\omega}^{2} = \mathcal{O}(N^{1/2-2s}).
  \end{equation}
\end{corollary}
\begin{proof}
  The proof is simply \cref{thm:useful_result_DPP} where the minimum of r.h.s. is upper-bounded by the value at $M = N^{1-1/4s}$. This yields the bound in \eqref{e:slargerthanonehalf}.
  Note that the larger value $M=N$ leads to a slower rate in $\mathcal{O}(N^{1-2s})$.  
\end{proof}
}

Finally, the following lemma, which is used in the proof of \Cref{thm:useful_result_DPP}, underlines a practical limitation of the least-squares approximation.
\begin{lemma}\label{prop:generic_formula_omega_norm}
Let $f \in \mathcal{F}$, $\bm{w} \in \mathbb{R}^{N}$ and $\bm{x} \in \mathcal{X}^N$. 
Then
\begin{equation}\label{eq:identity_epsilon_f}
  \|f-\sum_{i \in [N]}w_{i}k(x_i,\cdot)\|_{\omega}^2 = \|f\|_{\omega}^2 -2 \sum_{i \in [N]}w_{i}\bm{\Sigma}f(x_i) + \bm{w}^{\mathrm{T}} \bm{K}_{2}(\bm{x})\bm{w},
\end{equation}
where $\bm{K}_{2}(\bm{x})$ is defined in \cref{s:different_kernels}.
If $\bm{K}_{2}(\bm{x})$ is non-singular, the associated least-squares approximation is $\hat{f}_{\mathrm{LS},\bm{x}}=\sum_{i \in [N]}\hat{w}_{i}k(x_{i},\cdot)$ where 
\begin{equation}\label{eq:w_hat_using_k2}
\hat{\bm{w}} = \bm{K}_{2}(\bm{x})^{-1}\bm{\Sigma}f(\bm{x}).
\end{equation} 
\end{lemma}
As a consequence, the most direct way to evaluate $\hat{f}_{\mathrm{LS},\bm{x}}$ requires evaluating $\bm{\Sigma}f$ rather than $f$, and evaluating the kernel $k_{2}$ instead of $k$; see \cref{s:different_kernels}.
Both may not have tractable expressions, which can be an important limitation in practice. 
The proof of~\Cref{prop:generic_formula_omega_norm} is based on the Mercer decomposition~\eqref{eq:mercer_infinite_sum} and is given in~\Cref{proof:generic_formula_omega_norm}.

\rev{
  \subsection{Guarantees for optimal kernel approximation}\label{sec:oka}
}

\rev{To bypass the practical limitations of the least-squares approximation highlighted in \cref{prop:generic_formula_omega_norm}, we now study the optimal kernel approximation $\hat{f}_{\mathrm{OKA},\bm{x}}$, defined in~\eqref{eq:interpolation_opt_problem}.}

\begin{theorem}\label{thm:oka_bounds}
\rev{Let $f\in\mathcal{F}$ and $\bm{x}\in\mathcal{X}^N$. 
For any $M \in \mathbb{N}^{*}$, 
  \begin{equation}
    \label{e:first_OKA_result}
    \|f - \hat{f}_{\mathrm{OKA},\bm{x}}\|_{\omega}^2 \leq \|f\|_{\mathcal{F}}^{2} \, \Big(\sum\limits_{m \in [M]} \sigma_{m} \|e_{m}^{\mathcal{F}} - \widehat{e_{m}^{\mathcal{F}}} \|_{\mathcal{F}}^{2} + \sigma_{M+1} \Big),
  \end{equation}
  where, for conciseness, we wrote $\widehat{e_{m}^{\mathcal{F}}}=\Pi_{\mathcal{K}(\bm{x})}^{\|.\|_{\mathcal{F}}}e_{m}^{\mathcal{F}}$.  
  In particular, 
  \begin{equation}
    \label{e:second_OKA_result}
    \mathbb{E}_{\DPP} \|f - \hat{f}_{\mathrm{OKA},\bm{x}}\|_{\omega}^2 \leq  \|f\|_{\mathcal{F}}^{2}\inf\limits_{M \in [N]}\Big(4M r_{N+1} + \sigma_{M+1} \Big) ,
  \end{equation}
  and 
  \begin{equation}
    \label{e:third_OKA_result}
    \mathbb{E}_{\VS} \|f - \hat{f}_{\mathrm{OKA},\bm{x}}\|_{\omega}^2 \leq \|f\|_{\mathcal{F}}^{2}\inf\limits_{M \in \mathbb{N}^{*}}\Big(\sum\limits_{m \in [M]} \epsilon_{m}(N) + \sigma_{M+1} \Big).
\end{equation}
where $r_{N+1}$ is defined by \eqref{eq:def_r_Nplus1} and $\epsilon_{m}(N)$ is defined by \eqref{eq:def_epsilon}.
}
\end{theorem}

\rev{
  The proof is in \cref{proof:thm_oka_bounds}.
  \Cref{thm:oka_bounds} shows a better bound for CVS than under the projection DPP. 
  The following corollary translates into a faster convergence rate for CVS under the parametric assumption that the spectrum $(\sigma_m)$ decreases polynomially.
}
\begin{corollary}\label{cor:oka_polynomial_rate}
\rev{
  Assume $\sigma_{N+1} \leq C N^{-\alpha}$ with $\alpha>1$ and $C>0$. Then
  \begin{equation}
    \mathbb{E}_{\DPP} \|f-\hat{f}_{\mathrm{OKA}, \bm{x}}\|_{\omega}^2 = \mathcal{O}(N^{-\alpha\frac{\alpha-1}{\alpha+1}})
  \end{equation}
  and
\begin{equation}
  \mathbb{E}_{\mathrm{CVS}} \|f-\hat{f}_{\mathrm{OKA}, \bm{x}}\|_{\omega}^2 = \mathcal{O}(N^{-\alpha\frac{\alpha}{\alpha +1}}).
\end{equation}    
}
\end{corollary}
\begin{proof}
\rev{
  We have $r_{N+1} \leq C' N^{1-\alpha}$ with $C'>0$. 
  For any $\gamma\in(0,1)$, \eqref{e:second_OKA_result} with $M = \lfloor N^{\gamma} \rfloor$ guarantees that
  \begin{equation}
    \mathbb{E}_{\DPP} \|f-\hat{f}_{\mathrm{OKA}, \bm{x}}\|_{\omega}^2 \leq 4C' N^{\gamma + 1-\alpha} + C N^{-\gamma\alpha} \|f\|_{\mathcal{F}}^2.
  \end{equation}
Letting $\gamma = (\alpha-1)/(\alpha+1)$ leads to the announced convergence rate. 
}
\rev{
For CVS, we combine \eqref{CVS_eq:upper_bound_sup_epsilon} and \eqref{CVS_eq:ineq_r_N} to obtain
\begin{equation}
\sum\limits_{m \in [M]} \epsilon_{m}(N) \leq M \epsilon_{1}(N) \leq C'M \sigma_{N+1}.
\end{equation}
Again, taking $M = \lfloor N^{\gamma} \rfloor$ for some $\gamma\in (0,1)$ in \eqref{e:third_OKA_result}, we get
\begin{equation}
\mathbb{E}_{\VS} \|f-\hat{f}_{\mathrm{OKQ},M, \bm{x}}\|_{\omega}^2 \leq C' N^{\gamma -\alpha} + C N^{-\gamma\alpha} \|f\|_{\mathcal{F}}^2.
\end{equation}
Fixing $\gamma = \alpha/(\alpha+1)$ to its optimal value leads to the announced rate.}
% ...
\end{proof}
\rev{
  For both the projection DPP and CVS, the mean-squared error of OKA is thus close to $\mathcal{O}(\sigma_{N+1})$ under polynomial decay of the spectrum, with a slight edge for CVS in \cref{cor:oka_polynomial_rate}. 
  In particular, while for large values of $\alpha$ the two rates in \cref{cor:oka_polynomial_rate} are similar, the rate for CVS is preferable for $\alpha$ close to $1$.
  This matches the intuition that, if the spectrum decays slowly, there is an advantage in taking a mixture of projection determinantal point processes such as CVS, rather than taking a single projection DPP onto a set of $N$ eigenfunctions.
}

\rev{
  Finally, the proof of \Cref{thm:oka_bounds} does not yield an improved bound when we further assume $f = \bm{\Sigma}^{1/2+r}g$ and $r>0$. However, a different argument permits to treat at least the case $r=1/2$.}
\begin{proposition}\label{prop:f_sigmag_cvs_omega}
\rev{Assume that $f = \bm{\Sigma}g$, where $g \in \mathbb{L}_{2}(\omega)$. Then, we have
\begin{equation}
\mathbb{E}_{\DPP} \|f - \hat{f}_{\mathrm{OKA},\bm{x}} \|_{\omega}^{2} \leq 4 \|g\|_{\omega}^{2} \; r_{N+1} \sum\limits_{m \in \mathbb{N}^{*}} \sigma_m, \:\:
\end{equation}
and 
\begin{equation}
\mathbb{E}_{\VS} \|f - \hat{f}_{\mathrm{OKA},\bm{x}} \|_{\omega}^{2} \leq \|g\|_{\omega}^2 \; \epsilon_{1}(N) \sum\limits_{m \in \mathbb{N}^{*}} \sigma_m .
\end{equation}}
\end{proposition}
\rev{
  The proof is in \cref{proof:prop:f_sigmag_cvs_omega} and simply consists in bounding the $L^2$ norm of the residual by its RKHS norm, and using previous results on the RKHS norm.} 
\rev{
  Finally, observe that the theoretical bounds on OKA allow to derive the following lemma that provides an upper bounds for the transform based on optimal kernel quadrature.} 
\begin{lemma}\label{prop:transform_based_OKA}
Let $f \in \mathcal{F}$ and $M, N \in \mathbb{N}^*$. 
Given $\bm{x} \in \mathcal{X}^{N}$ such that the matrix $\bm{K}_{1}(\bm{x})$ is non-singular, let $\hat{f}_{\mathrm{OKQ},M,\bm{x}}$ the approximation \eqref{def:hat_phi_x} with weights $\bm{\alpha}_{m}:=(\alpha_{m,i})$ defined by~\eqref{eq:OKQ_alpha_n}.
Then
\begin{equation}\label{eq:OKQ_transform_OKA_bound}
\forall M \in \mathbb{N}^*, \:\: \|f-\hat{f}_{\mathrm{OKQ},M, \bm{x}}\|_{\omega}^2 \leq \|f-\hat{f}_{\mathrm{OKA},\bm{x}}\|_{\omega}^2 + \|f-f_{M}\|_{\omega}^2,
\end{equation}
where $f_M := \sum_{m\in [M]} \langle f, e_m\rangle_{\omega} e_m$. 
\end{lemma}
\rev{Note that the upper bound \eqref{eq:OKQ_transform_OKA_bound} is not multiplicative in the squared residual $\|f-f_{M}\|_{\omega}^2$. 
Therefore the IOP\footnote{Defined in \Cref{sec:iid_design_conf} for the empirical least square approximation $\hat{f}_{\mathrm{ELS},\bm{x}}$, but it can be extended to any finite-dimensional approximation.} does not hold a priori. \Cref{sec:bounds_pi} will show that the quasi-interpolant approximation satisfies the IOP under the distribution of the projection DPP.}
\subsection{The least squares approximation under CVS}\label{sec:ls_cvs}
\rev{Going back to the least squares approximation, \Cref{thm:oka_bounds} and \Cref{prop:f_sigmag_cvs_omega} permit to prove the following result.}
\rev{\begin{proposition}
    \label{cor:cvsrates}
    Let $r \geq 0$ and $f \in \bm{\Sigma}^{1/2+r}\mathbb{L}_{2}(\omega)$.
    Then 
    \begin{equation}
      \label{e:rnonnegative}
      \mathbb{E}_{\VS}\|f -\hat{f}_{\mathrm{LS},\bm{x}} \|_{\omega}^{2} \leq \|f\|_{\mathcal{F}}^{2}\inf\limits_{M \in \mathbb{N}^{*}}\Big(\sum\limits_{m \in [M]} \epsilon_{m}(N) + \sigma_{M+1} \Big).
    \end{equation}
    Finally, if $r\geq 1/2$, 
    \begin{equation}
      \label{e:rlargerthanonehalf}
      \mathbb{E}_{\VS}\|f -\hat{f}_{\mathrm{LS},\bm{x}} \|_{\omega}^{2}  \leq \|g\|_{\omega}^2 \; \epsilon_{1}(N) \sum\limits_{m \in \mathbb{N}^{*}} \sigma_m.
    \end{equation}
  \end{proposition}}
\begin{proof}
\rev{Let $f \in \mathcal{F}$. By definition of $\hat{f}_{\mathrm{LS},\bm{x}}$, we have 
\begin{equation}
\|f-\hat{f}_{\mathrm{LS},\bm{x}}\|_{\omega}^{2} \leq \|f-\hat{f}_{\mathrm{OKA},\bm{x}}\|_{\omega}^{2}. 
\end{equation}
Thus, 
\begin{equation}
\mathbb{E}_{\VS} \|f-\hat{f}_{\mathrm{LS},\bm{x}}\|_{\omega}^{2} \leq \mathbb{E}_{\VS}\|f-\hat{f}_{\mathrm{OKA},\bm{x}}\|_{\omega}^{2}.
\end{equation}
We conclude by using \Cref{thm:oka_bounds} and \Cref{prop:f_sigmag_cvs_omega} proven later.}
\end{proof}
\rev{For least squares and assuming $r\geq 1/2$, our bound under the DPP in \cref{cor:dpprates} is lower than \eqref{e:rlargerthanonehalf}. 
However, when the sequence $(\sigma_m)$ decreases polynomially, we recover the same rates of convergence as under the DPP in \Cref{cor:tels_polynomial_rate}.}

\subsection{The instance optimality property under projection DPPs}\label{sec:bounds_pi}

This section investigates the IOP \eqref{eq:IOP} of the empirical least-squares approximation \eqref{def:hat_f_ELS_M} under the DPP defined in~\Cref{s:DPPs}. 
More precisely, it shows that under the DPP, the IOP is satisfied by a variant of the empirical least-squares approximation, with a smaller budget than i.i.d. Christoffel sampling. 

Given a positive function $q: \mathcal{X} \rightarrow \mathbb{R}_{+}^{*}$, and given $M, N \in \mathbb{N}^{*}$ such that $M \leq N$, we define the \emph{truncated empirical least-squares approximation of order $M$}  associated to the configuration $\bm{x} \in \mathcal{X}^N$ as
\begin{equation}\label{eq:def_telsa}
  \hat{f}_{\mathrm{tELS},M, \bm{x}} := \rev{\Pi_{\mathcal{E}_{M}}^{\|.\|_{\omega}} \hat{f}_{\mathrm{ELS}, N, \bm{x}}} =  \sum\limits_{m \in [M]} \langle \hat{f}_{\mathrm{ELS},N, \bm{x}}, e_{m} \rangle_{\omega} e_{m}, 
\end{equation} 
where $\hat{f}_{\mathrm{ELS},N, \bm{x}}$ is the empirical least-squares approximation of $f$ of dimension $N$ associated to the function $q$ and the configuration $\bm{x}$ of size $N$. 
A subtle but important remark is that, in general, $\hat{f}_{\mathrm{tELS},M, \bm{x}} \neq \hat{f}_{\mathrm{ELS},M, \bm{x}}$.
\rev{Indeed, $\hat{f}_{\mathrm{ELS},M, \bm{x}}$ is the ELS approximation in ${\cal E}_M$, while $\hat{f}_{\mathrm{tELS},M, \bm{x}}$ is the projection on ${\cal E}_M$ of the ELS approximation in ${\cal E}_N$.
Interestingly, it will appear that} $\hat{f}_{\mathrm{tELS},M, \bm{x}}$ is even invariant to the choice of the function $q$, which is not the case of $\hat{f}_{\mathrm{ELS},M, \bm{x}}$. 






\begin{lemma}\label{prop:ELS_is_PI} Let $N \in \mathbb{N}^{*}$ and $\bm{x} \in \mathcal{X}^{N}$ be such that the matrix \rev{$\bm{K}_{0,N}(\bm{x})$} is non-singular.
Then
\begin{equation}
  \hat{f}_{\mathrm{ELS},N,\bm{x}} = \hat{f}_{\mathrm{QI},N,\bm{x}},
\end{equation}
where $\hat{f}_{\mathrm{QI},N,\bm{x}}$ is the quasi-interpolant \eqref{eq:def_fQI} defined in~\Cref{sec:eigenspace_approximations}.
As a result, 
\begin{equation}
  \forall M \in [N], \:\: \hat{f}_{\mathrm{tELS},M,\bm{x}} = \sum\limits_{m\in[M]} \langle \hat{f}_{\mathrm{QI},N,\bm{x}}, e_{m} \rangle_{\omega}e_{m}.
\end{equation}
\end{lemma}
\rev{The proof of~\Cref{prop:ELS_is_PI} is given in~\Cref{proof:ELS_is_PI}.
As a consequence, the empirical least-squares approximation of dimension $N$ coincides with the quasi-interpolant \eqref{eq:def_fQI} and is invariant to the choice of the function $q$. 
Moreover,} $\hat{f}_{\mathrm{tELS},M,\bm{x}}$
is simply the projection of $\hat{f}_{\mathrm{QI}, N,\bm{x}}$ onto the eigenspace $\mathcal{E}_{M}$. 
Therefore, the numerical evaluation of $\hat{f}_{\mathrm{tELS},M,\bm{x}}$ boils down to the evaluation of the quadrature rules \eqref{eq:pseudo_GQ}, which in turn require the evaluation of the weights \eqref{eq:alpha_QI}.
\rev{This involves an $N \times N$ rather than an $M \times M$ matrix only for the evaluation of~\eqref{eq:ELS_matrix_form} to get $\hat{f}_{\mathrm{ELS},M,\bm{x}}$.} 
Therefore the evaluation of the coefficients of~$\hat{f}_{\mathrm{tELS},M,\bm{x}}$ in the basis $(e_{m})_{m \in \mathbb{N^*}}$ is numerically more expensive than $\hat{f}_{\mathrm{ELS},M,\bm{x}}$. 
\rev{This is the price to pay for this new approximation to be more amenable to analysis, as shown in the following \cref{prop:bound_tELS_DPP} when the configuration $\bm{x}$ is a projection DPP.}
\begin{proposition}\label{prop:bound_tELS_DPP} Consider $M, N \in \mathbb{N}^{*}$ such that $M \leq N$. Then, for $f \in \mathcal{F}$
\begin{equation}\label{eq:EX_DPP_ELS_N0}
\mathbb{E}_{\DPP} \|f-\hat{f}_{\mathrm{tELS},M,\bm{x}}\|_{\omega}^2 = \|f-f_{M}\|_{\omega}^2 +  M\|f- f_{N}\|_{\omega}^{2},
\end{equation}
As a direct consequence,
\begin{equation}\label{eq:f_tels_bound_2}
\mathbb{E}_{\DPP} \|f-\hat{f}_{\mathrm{tELS},M,\bm{x}}\|_{\omega}^2 \leq (1+M)\|f-f_{M}\|_{\omega}^2.
\end{equation}
\end{proposition}
Under the projection DPP, $\hat{f}_{\mathrm{tELS},M,\bm{x}}$ thus satisfies the IOP \eqref{eq:IOP} with constant $(1+M)$, for any sampling budget that satisfies $N \geq M$. 
In particular, the IOP holds for a smaller budget than the $O(M \log(M))$ required by i.i.d. Christoffel sampling in \citep{CoMi17}; see \cref{sec:iid_design_conf}. \rev{Observe that \eqref{eq:f_tels_bound_2} is an equality when $N = M$, by \eqref{eq:EX_DPP_ELS_N0}. In that case, $\hat{f}_{\mathrm{tELS},M,\bm{x}} = \hat{f}_{\mathrm{QI},N,\bm{x}}$.
The proof of~\Cref{prop:bound_tELS_DPP}, based on the identities~\eqref{eq:ez_1} and \eqref{eq:ez_2}, is given in~\Cref{proof:bound_tELS_DPP}.}

\rev{
Finally, we derive the rate of convergence of $\hat{f}_{\mathrm{tELS},M,\bm{x}}$ to $f$ under the parametric assumption that the spectrum $(\sigma_m)$ decreases polynomially.
}
\begin{corollary}\label{cor:tels_polynomial_rate}
\rev{
  Assume $\sigma_{N+1} \leq C N^{-\alpha}$ with $\alpha>1$ and $C>0$. Then, taking $M= N^{\alpha/(\alpha+1)}$ in \cref{prop:bound_tELS_DPP}, we get
  $$
    \mathbb{E}_{\DPP} \|f-\hat{f}_{\mathrm{tELS},M,\bm{x}}\|_{\omega}^2 = \mathcal{O}(N^{-\alpha\frac{\alpha}{\alpha+1}}).
  $$
}
\end{corollary}
\rev{
  The proof of \Cref{cor:tels_polynomial_rate} follows the same steps as the proof of \Cref{cor:oka_polynomial_rate}.
  Interestingly, \Cref{cor:tels_polynomial_rate} yields a faster rate of convergence for $\hat{f}_{\mathrm{tELS},M,\bm{x}}$ than \Cref{cor:oka_polynomial_rate} did for $\hat{f}_{\mathrm{OKA},\bm{x}}$. 
  }




% ...



\subsection{Proofs}\label{sec:proofs}

\subsubsection{Proof of \Cref{prop:generic_formula_omega_norm}}
\label{proof:generic_formula_omega_norm}
Let $f \in \mathcal{F}$, $\bm{w} \in \mathbb{R}^{N}$ and $\bm{x} \in \mathcal{X}^N$. 
The squared residual $\|f-\sum_{i \in [N]}w_{i}k(x_i,\cdot)\|_{\omega}^2$ writes 
\begin{equation}
  \label{e:identity}
  \|f\|_{\omega}^2 -2 \langle f, \sum_{i \in [N]}w_{i}k(x_i,\cdot) \rangle_{\omega} + \|\sum_{i \in [N]}w_{i}k(x_i,\cdot)\|_{\omega}^2.
\end{equation}
The identity~\eqref{eq:identity_epsilon_f} follows from evaluating the latter two terms. 
First note that 
\begin{align}
  \langle f, k(x_{i},\cdot) \rangle_{\omega} = \bm{\Sigma}f(x_i),
  \label{e:tool1}
\end{align}
so that the linear term in \eqref{eq:identity_epsilon_f} is as claimed. 
Now, the uniform convergence in Mercer's decomposition \eqref{eq:mercer_infinite_sum} allows us to write, for $i, j\in [N]^2$,
\begin{align}
  \label{e:tool2}
  \langle k(x_{i},\cdot),k(x_{j},\cdot) \rangle_{\omega} &= \int \lim_{n\rightarrow \infty} \left(\sum_{m \in [n]} \sigma_m e_m(x_{i}) e_m(y) \right) \left(\sum_{m \in [n]} \sigma_m e_m(x_{j}) e_m(y) \right) \mathrm{d}\omega(y)\\
  &=  k_{2}(x_i,x_{j})
\end{align}
by dominated convergence. 
Plugging \eqref{e:tool1} and \eqref{e:tool2} into \eqref{e:identity} yields the desired formula.

\subsubsection{Proof of \Cref{prop:f_sigmag_cvs_omega}}
\label{proof:prop:f_sigmag_cvs_omega}
\rev{
  We use again
  \begin{align}\label{eq:OKA_bound_using_em_other_proof}
  \|f-\hat{f}_{\mathrm{OKA},\bm{x}}\|_{\omega}^2 & = \sum\limits_{m \in \mathbb{N}^{*}} \langle f - \hat{f}_{\mathrm{OKA},\bm{x}}, e_{m} \rangle_{\omega}^2.
  \end{align}
  On the one hand, since $f = \bm{\Sigma}g$,
  $$
    \langle f, e_m \rangle_{\omega} = \langle \bm{\Sigma}g, e_m \rangle_{\omega} = \langle g, \bm{\Sigma} e_m \rangle_{\omega} =  \sigma_m \langle g, \bm{\Sigma} e_m \rangle_{\F}.
  $$
  On the other hand, by definition, 
  letting $\bm{w} = \bm{K}_1(\bm{x})^{-1}f(\bm{x})$,
  \begin{equation}
  \langle \hat{f}_{\mathrm{OKA},\bm{x}}, e_{m} \rangle_{\omega} = \sigma_m \langle \hat{f}_{\mathrm{OKA},\bm{x}}, e_{m} \rangle_{\mathcal{F}} = \sigma_{m}\sum\limits_{i\in [N]}w_{i}e_{m}(x_i) = \sum\limits_{i\in [N]}w_{i}\bm{\Sigma}e_{m}(x_i),
  \end{equation}
  Thus
  \begin{equation}
  \langle f,e_{m}\rangle_{\omega} - \langle \hat{f}_{\mathrm{OKA},\bm{x}}, e_{m} \rangle_{\omega} = \langle g, \bm{\Sigma}e_m \rangle_{\omega} - \sum\limits_{i\in [N]}w_{i} \bm{\Sigma}e_{m}(x_i).
  \end{equation}
  Now, 
  applying \eqref{eq:quadrature_error_upper_bound} with $h=\bm{\Sigma}e_m \in \mathcal{F}$, it comes
  \begin{align}
    \langle f - \hat{f}_{\mathrm{OKA},\bm{x}}, e_m\rangle_{\omega}^2 
    &\leq  \|\bm{\Sigma}e_m\|_{\mathcal{F}}^2 \|f- \sum\limits_{i \in [N]}w_i k(x_i,.)\|_{\mathcal{F}}^2 \\
    &\leq  \sigma_{m}\|f-\hat{f}_{\mathrm{OKA},\bm{x}}\|_{\mathcal{F}}^2.
  \end{align}
  Therefore
  \begin{equation}
  \|f-\hat{f}_{\mathrm{OKA},\bm{x}}\|_{\omega}^2 \leq \Big(\sum\limits_{m \in \mathbb{N}^{*}} \sigma_{m}\Big)\|f-\hat{f}_{\mathrm{OKA},\bm{x}}\|_{\mathcal{F}}^2.
  \end{equation}
  We conclude by using \eqref{eq:result_ezq_2}, \eqref{CVS_eq:main_result_EX_VS_err_mu} and \eqref{CVS_eq:upper_bound_sup_epsilon}.
}

\subsubsection{Proof of 
\Cref{thm:useful_result_DPP}}\label{proof:thm_useful_result_DPP}
First, observe that when $\bm{x}$ follows the distribution of the projection DPP of~\Cref{def:dpp}, the Gram matrix \rev{$\bm{K}_{0,N}(\bm{x})$} associated with the kernel $k_{0,N}$ defined in \cref{s:different_kernels} is almost surely non-singular. 
Similarly, the Gram matrix $\bm{K}_{2,N}(\bm{x})$ associated to the kernel $k_{2,N}$ is almost surely non-singular. 
Now, observe that for any $\bm{x}$,
\begin{equation}
\bm{K}_{2,N}(\bm{x}) = \sum\limits_{m \in [N]} \sigma_{m}^2e_{m}(\bm{x})e_{m}(\bm{x})^{\mathrm{T}} \prec \bm{K}_{2}(\bm{x}) = \sum\limits_{m \in \mathbb{N}^{*}} \sigma_{m}^2e_{m}(\bm{x})e_{m}(\bm{x})^{\mathrm{T}}, 
\end{equation}
so that the matrix $\bm{K}_{2}(\bm{x})$ is also almost surely non-singular.

\rev{
Let $f \in \F$, and $\bm{x} \in \mathcal{X}^N$ be such that the matrices \rev{$\bm{K}_{0,N}(\bm{x}),\bm{K}_{2,N}(\bm{x})$} are non-singular. In short, we use an approximation inspired by the Ermakov-Zolotukhin quadrature of \Cref{sec:review_ezq} as a proxy for the least-squares approximation. 
Define 
$$
  \hat{f}_{\bm{x}}:= \sum_{i\in [N]}\hat{w}_{i}k(x_i,\cdot),
$$
where $\bm{\hat{\bm{w}}} := \bm{K}_{1,N}(\bm{x})^{-1}f_{M}(\bm{x})$ and
$f_{N}$ is defined by~\eqref{eq:eigen_approximation}. 
By definition of $\hat{f}_{\mathrm{LS},\bm{x}}$, 
$\|f -\hat{f}_{\mathrm{LS},\bm{x}} \|_{\omega}^{2}  \leq \|f -\hat{f}_{\bm{x}} \|_{\omega}^{2}$, so that, letting $M\leq N$, 
\begin{equation}
  \|f -\hat{f}_{\mathrm{LS},\bm{x}} \|_{\omega}^{2} \leq  2 \Big(\|f-f_{M}\|_{\omega}^2 + \|f_{M} -\hat{f}_{\bm{x}}\|_{\omega}^{2} \Big).
  \label{e:centraltool}
\end{equation}
Now, using \Cref{prop:generic_formula_omega_norm},
\begin{equation}
  \label{eq:identity_epsilon_fN}
  \rev{ \|f_{M} - \hat{f}_{\bm{x}}\|_{\omega}^2 = \|f_M\|_{\omega}^2 -2 \sum_{i \in [N]}\hat{w}_{i}\bm{\Sigma}f_M(x_i) + \hat{\bm{w}}^{\mathrm{T}} \bm{K}_{2}(\bm{x})\hat{\bm{w}}.}
\end{equation}
In the following, we shall prove that 
\begin{equation}\label{eq:first_term_in_error_ez}
\mathbb{E}_{\mathrm{DPP}} \sum_{i \in [N]}\hat{w}_{i}\bm{\Sigma}f_M(x_i) = \|f_M\|_{\omega}^2,
\end{equation}
and
\begin{equation}\label{eq:second_term_in_error_ez}
\mathbb{E}_{\mathrm{DPP}} \hat{\bm{w}}^{\mathrm{T}} \bm{K}_{2}(\bm{x})\hat{\bm{w}} = \|f_M\|_{\omega}^2 + \Big(\sum\limits_{m\in[M]} \sigma_{m}^{-2}\langle f, e_m \rangle_{\omega}^2  \Big) \sum\limits_{m \geq N+1} \sigma_{m}^2.
\end{equation}
This will be enough to conclude, since \eqref{eq:identity_epsilon_fN} will then yield 
\begin{equation}
\mathbb{E}_{\mathrm{DPP}} \|f_{M} - \hat{f}_{\bm{x}} \|_{\omega}^2 = \Big(\sum\limits_{m\in[M]} \sigma_{m}^{-2}\langle f, e_m \rangle_{\omega}^2  \Big) \sum\limits_{m \geq N+1} \sigma_{m}^2,
\end{equation}
which, plugged into \eqref{e:centraltool}, shall conclude the proof.
}

\rev{We start with the proof of \eqref{eq:first_term_in_error_ez}. First, observe that 
\begin{equation}
  \bm{\Sigma}f_{M}(x_i) = \sum\limits_{m\in[M]} \sigma_{m} \langle f, e_{m} \rangle_{\omega} e_{m}(x_i).
\end{equation}
Therefore
\begin{equation}\label{eq:sum_w_ez_Sigmafn}
  \sum\limits_{i \in [N]} \hat{w}_{i}\bm{\Sigma}f_{M}(x_i) = \sum\limits_{m\in[M]} \sigma_{m} \langle f, e_{m} \rangle_{\omega} \sum\limits_{i \in [N]} \hat{w}_{i}e_{m}(x_i)
\end{equation}
By Proposition 2 in \citep{Bel21}, we have
\begin{equation}
  \hat{\bm{w}} = \rev{\bm{K}_{0,N}(\bm{x})}^{-1} (h_M(\bm{x})),
\end{equation}
where, $h_M(x) := \sum_{m \in [N]}\sigma_{m}^{-1} \langle f,e_{m} \rangle_{\omega} e_{m}(x)$, so that
\begin{equation}
  \sum\limits_{i \in [N]} \hat{w}_{i}e_{m}(x_i) =  \sum\limits_{n \in [M]} \sigma_{n}^{-1} \langle f, e_{n} \rangle_{\omega} e_{n}(\bm{x})^{\mathrm{T}}\rev{\bm{K}_{0,N}(\bm{x})}^{-1} e_{m}(\bm{x}).
\end{equation}
Thus, 
\begin{align}
\mathbb{E}_{\DPP} \sum\limits_{i \in [N]} \hat{w}_{i}e_{m}(x_i) & = \sum\limits_{n \in [M]} \sigma_{n}^{-1} \langle f, e_{n} \rangle_{\omega} \mathbb{E}_{\DPP}   e_{n}(\bm{x})^{\mathrm{T}}\rev{\bm{K}_{0,N}(\bm{x})}^{-1} e_{m}(\bm{x}) \nonumber.
\end{align}
By the unbiasedness property \eqref{eq:ez_1}, and since $M \leq N$, we obtain
\begin{align}\label{eq:E_DPP_w_hat}
  \mathbb{E}_{\DPP} \sum\limits_{i \in [N]} \hat{w}_{i}e_{m}(x_i) & = \sum\limits_{n \in [M]} \sigma_{n}^{-1} \langle f, e_n \rangle_{\omega} \delta_{n,m} = \sigma_{m}^{-1} \langle f, e_m \rangle_{\omega}.
\end{align}
Combining \eqref{eq:sum_w_ez_Sigmafn} and \eqref{eq:E_DPP_w_hat}, we get \eqref{eq:first_term_in_error_ez}.
Now we move to the proof of \eqref{eq:second_term_in_error_ez}. For this purpose, write 
\begin{equation}
\bm{K}_{2}(\bm{x})  = \bm{K}_{2,N}(\bm{x}) + \bm{K}_{2,N}^{\perp}(\bm{x}), 
\end{equation}
where 
\begin{equation}
\bm{K}_{2,N}^{\perp}(\bm{x}):= \sum\limits_{m \geq N+1} \sigma_{m}^{2} e_{m}(\bm{x}) e_{m}(\bm{x})^{\mathrm{T}}. 
\end{equation}
In particular, we have
\begin{equation}
\hat{\bm{w}}^\mathrm{T} \bm{K}_{2}(\bm{x}) \hat{\bm{w}} = \hat{\bm{w}}^{\mathrm{T}} \bm{K}_{2,N}(\bm{x}) \hat{\bm{w}} + \hat{\bm{w}}^{\mathrm{T}} \bm{K}_{2,N}^{\perp}(\bm{x}) \hat{\bm{w}}.
\end{equation}
Now, using Proposition 2 in \citep{Bel21}, we have
\begin{equation}
\hat{\bm{w}} = \rev{\bm{K}_{0,N}(\bm{x})}^{-1} h_M(\bm{x}) = \bm{K}_{2,N}(\bm{x})^{-1} \tilde{h}_M(\bm{x}),
\end{equation}
where
\begin{equation}
\tilde{h}_M(x):= \sum\limits_{m\in[M]} \sigma_{m} \langle f, e_{m} \rangle_{\omega} e_{m}(x). 
\end{equation}
Thus
\begin{equation}
\hat{\bm{w}}^{\mathrm{T}} \bm{K}_{2,N}(\bm{x}) \hat{\bm{w}} =  \tilde{h}_M(\bm{x})^{\mathrm{T}} \bm{K}_{2,M}(\bm{x})^{-1} \tilde{h}_M(\bm{x})= \tilde{h}_M(\bm{x})^{\mathrm{T}} \hat{\bm{w}} = \tilde{h}_M(\bm{x})^{\mathrm{T}} \rev{\bm{K}_{0,N}(\bm{x})}^{-1} h_M(\bm{x}). 
\end{equation}
Therefore
\begin{align}
\mathbb{E}_{\DPP} \hat{\bm{w}}^{\mathrm{T}} \bm{K}_{2,N}(\bm{x}) \hat{\bm{w}} &  = \mathbb{E}_{\DPP} \tilde{h}_M(\bm{x})^{\mathrm{T}} \rev{\bm{K}_{0,N}(\bm{x})}^{-1} h_M(\bm{x}) \nonumber \\
& = \sum\limits_{n\in [M]} \sum\limits_{n' \in [M]} \sigma_{n}\langle f, e_{n} \rangle_{\omega} \sigma_{n'}^{-1} \langle f, e_{n'} \rangle_{\omega} \mathbb{E}_{\DPP} e_{n}(\bm{x})^{\mathrm{T}} \rev{\bm{K}_{0,N}(\bm{x})}^{-1} e_{n'}(\bm{x}) \nonumber \\
& = \sum\limits_{n\in [M]} \sum\limits_{n' \in [M]} \sigma_{n}\langle f, e_{n} \rangle_{\omega} \sigma_{n'}^{-1} \langle f, e_{n'} \rangle_{\omega} \delta_{n,n'} \nonumber \\
& = \sum\limits_{n\in [M]} \frac{\langle f, e_{n} \rangle_{\omega}^2}{\sigma_{n}^2}.
\end{align}
Now we prove that 
\begin{equation}
\mathbb{E}_{\DPP} \hat{\bm{w}}^{\mathrm{T}} \bm{K}_{2,N}^{\perp}(\bm{x}) \hat{\bm{w}} = \|f_{M}\|_{\omega}^2 \sum\limits_{m \geq N+1} \sigma_{m}^{2}.
\end{equation}
For that purpose, observe that 
\begin{equation}
\hat{\bm{w}}^{\mathrm{T}} \bm{K}_{2,N}^{\perp}(\bm{x}) \hat{\bm{w}} = \sum\limits_{m \geq N+1} \sigma_{m}^{2} \Big(\sum\limits_{i \in [N]} \hat{w}_{i} e_{m}(x_i) \Big)^2,
\end{equation}
and since $\langle h_M, e_{m} \rangle_{\omega} = 0$ for $m \geq N+1$, we have by \eqref{eq:ez_1} and \eqref{eq:ez_2}
\begin{align}
\mathbb{E}_{\DPP}\Big(\sum\limits_{i \in [N]} \hat{w}_{i} e_{m}(x_i) \Big)^2 & = \mathbb{E}_{\DPP} \Big(e_{m}(\bm{x}) \rev{\bm{K}_{0,N}(\bm{x})}^{-1} h_M(\bm{x})\Big)^2
\nonumber \\
& = \sum\limits_{n \in [M]} \sum\limits_{n' \in [M]} \sigma_{n}^{-1} \sigma_{n'}^{-1} \langle f, e_{n} \rangle_{\omega} \langle f, e_{n'} \rangle_{\omega} \mathbb{E}_{\DPP} \hat{I}_{n}(e_{m}) \hat{I}_{n'}(e_{m})   \\
& = \sum\limits_{n \in [M]} \frac{1}{\sigma_{n}^{2}} \langle f, e_{n} \rangle_{\omega}^2,
\end{align}
where the quadrature rule $\hat{I}_{n}$ is defined in \Cref{sec:review_ezq}.
}

\subsubsection{Proof of \Cref{thm:oka_bounds}}\label{proof:thm_oka_bounds}
\rev{
Let $M \in \mathbb{N}^{*}$, $\bm{x}\in\mathcal{X}^N$, and $f\in\mathcal{F}$. We have 
\begin{align}\label{eq:OKA_bound_using_em}
\|f-\hat{f}_{\mathrm{OKA},\bm{x}}\|_{\omega}^2 & = \sum\limits_{m \in \mathbb{N}^{*}} \langle f - \hat{f}_{\mathrm{OKA},\bm{x}}, e_{m} \rangle_{\omega}^2 \nonumber \\
& = \sum\limits_{m \in [M]} \langle f - \hat{f}_{\mathrm{OKA},\bm{x}}, e_{m} \rangle_{\omega}^2 + \sum\limits_{m \geq M+1} \frac{\sigma_m}{\sigma_m} \langle f - \hat{f}_{\mathrm{OKA},\bm{x}}, e_{m} \rangle_{\omega}^2 \nonumber \\
& \leq \sum\limits_{m \in [M]} \langle f - \hat{f}_{\mathrm{OKA},\bm{x}}, e_{m} \rangle_{\omega}^2 + \sigma_{M+1}\|\Pi_{\mathcal{E}_{M}^{\perp}}^{\mathcal{F}} \Pi_{\mathcal{K}(\bm{x})^{\perp}}^{\mathcal{F}} f \|_{\mathcal{F}}^{2}\nonumber \\
& \leq \sum\limits_{m \in [M]} \langle f - \hat{f}_{\mathrm{OKA},\bm{x}}, e_{m} \rangle_{\omega}^2 + \sigma_{M+1}\| f \|_{\mathcal{F}}^{2}. 
\end{align}
}
\rev{
Now, by definition of $\Vert\cdot\Vert_\mathcal{F}$ and because $\Pi_{\mathcal{K}(\bm{x})}^{\|.\|_{\mathcal{F}}}$ is self-adjoint, 
\begin{align}
\langle f - \hat{f}_{\mathrm{OKA},\bm{x}}, e_{m} \rangle_{\omega} & = \sqrt{\sigma_m}  \langle f - \Pi_{\mathcal{K}(\bm{x})}^{\|.\|_{\mathcal{F}}} f, e_{m}^{\mathcal{F}}\rangle_{\mathcal{F}}
= \sqrt{\sigma_m}  \langle f, e_{m}^{\mathcal{F}} - \widehat{e_{m}^{\mathcal{F}}}\rangle,
\label{e:rewriting}
\end{align}
where we write $\widehat{e_{m}^{\mathcal{F}}}$ for $\Pi_{\mathcal{K}(\bm{x})}^{\|.\|_{\mathcal{F}}} e_{m}^{\mathcal{F}}$.
Plugging \eqref{e:rewriting} into \eqref{eq:OKA_bound_using_em} and using Cauchy-Schwarz yields \eqref{e:first_OKA_result}.
}

\rev{
Now, monotone convergence and \eqref{eq:result_ezq_2} yield
\begin{equation}
  \mathbb{E}_{\DPP} \sum\limits_{m \in [M]} \sigma_{m} \| e_{m}^{\mathcal{F}} -  \widehat{e_{m}^{\mathcal{F}}}_{\mathrm{OKA}, \bm{x}} \|_{\mathcal{F}}^2 \leq 4M r_{N+1},
\end{equation}
proving \eqref{e:second_OKA_result}. 
Similarly, monotone convergence and \eqref{CVS_eq:main_result_EX_VS_err_mu} yield the equality
\begin{equation}
\mathbb{E}_{\VS} \sum\limits_{m \in [M]} \sigma_{m} \| e_{m}^{\mathcal{F}} -  \widehat{e_{m}^{\mathcal{F}}}_{\mathrm{OKA}, \bm{x}} \|_{\mathcal{F}}^2 = \sum\limits_{m \in [M]} \epsilon_{m}(N),
\end{equation}
proving \eqref{e:third_OKA_result}.}

\subsubsection{Proof of~\Cref{prop:transform_based_OKA}}\label{proof:transform_based_OKA}

Let $f \in \F$. \rev{By \eqref{eq:representation_OKQ_OKA}, we have
\begin{equation}
\hat{f}_{\mathrm{OKQ},M,\bm{x}} = \Pi_{\mathcal{E}_{M}}^{\|.\|_{\omega}}\hat{f}_{\mathrm{OKA},\bm{x}}.
\end{equation}
Therefore
\begin{align}
\|f-\hat{f}_{\mathrm{OKQ},M,\bm{x}}\|_{\omega}^2 = \|f-\Pi_{\mathcal{E}_{M}}^{\|.\|_{\omega}}\hat{f}_{\mathrm{OKA},\bm{x}}\|_{\omega}^2 & = \|f- f_{M} \|_{\omega}^2 + \|\Pi_{\mathcal{E}_M}^{\|.\|_{\omega}}(f- \hat{f}_{\mathrm{OKA},\bm{x}})\|_{\omega}^2 \nonumber\\
& \leq \|f- f_{M} \|_{\omega}^2 + \|f- \hat{f}_{\mathrm{OKA},\bm{x}}\|_{\omega}^2. 
\end{align}}

\subsubsection{Proof of~\Cref{prop:ELS_is_PI}}\label{proof:ELS_is_PI}
Let $f \in \mathcal{F}$, $\bm{x} = (x_{1}, \dots, x_{N}) \in \mathcal{X}^{N}$, and $\bm{w} = (w_{1}, \dots, w_{N}) \in \mathbb{R}^{N}$. Define
 \begin{equation}\label{eq:gamma_and_delta}
\left\{
    \begin{array}{ll}
    	\beta(\bm{x}) &:= \frac{1}{N}\sum_{i \in [N]}q(x_i)f(x_i)^2 \\
        \gamma(\bm{x}) & := (\frac{1}{N}\sum_{i \in [N]} q(x_{i})f(x_{i})e_{n}(x_{i}))_{n \in [N]} \in \mathbb{R}^{N}\\
        \delta(\bm{x}) & := (\frac{1}{N}\sum_{i \in [N]}q(x_{i}) e_{n_{1}}(x_{i}) e_{n_{2}}(x_{i}))_{(n_{1},n_{2}) \in [N]\times [N]} \in \mathbb{R}^{N \times N}.
    \end{array}
\right.
\end{equation}
We have
\begin{align}
\|f - \sum\limits_{m \in [N]} w_{m}e_{m} \|_{q,\bm{x}}^2 & = \frac{1}{N} \sum\limits_{i \in [N]}q(x_{i})\Big(f(x_{i}) - \sum\limits_{m \in [N]} w_{m}e_{m}(x_{i})\Big)^2\\
& = \beta(\bm{x})- 2\bm{w}^{\mathrm{T}}\gamma(\bm{x}) + \bm{w}^{\mathrm{T}}\delta(\bm{x})\bm{w}.
\end{align}
Under the assumption that $\delta(\bm{x})$ is non-singular, \rev{the solution of \eqref{def:hat_f_ELS_M} (when $M= N$) is unique.} \rev{Since $\hat{f}_{\mathrm{QI},N,\bm{x}}$ interpolates $f$ at the points $x_{1}, \dots, x_{N}$, it achieves the minimal value, $0$, in \eqref{def:hat_f_ELS_M}. Thus $\hat{f}_{\mathrm{QI},N,\bm{x}} = \hat{f}_{\mathrm{ELS},N,\bm{x}}$. }

Finally, we need to check that $\delta(\bm{x})$ is non-singular if and only if $\rev{\bm{K}_{0,N}(\bm{x})}$ is non-singular. This claim can be proved by observing that 
\begin{equation}
\delta(\bm{x})  = \frac{1}{N}E_{N}(\bm{x})^{\mathrm{T}}Q(\bm{x})E_{N}(\bm{x}),
\end{equation}
where 
\begin{equation}
E_{N}(\bm{x}):= (e_n(x_i))_{(i,n) \in [N] \times [N]} \in \mathbb{R}^{N \times N},
\end{equation}
so that $\delta(\bm{x})$ is non-singular if and only if the matrices $Q(\bm{x})$ and $E_{N}(\bm{x})$ are non-singular, and that $\rev{\bm{K}_{0,N}(\bm{x})} = E_{N}(\bm{x}) E_{N}(\bm{x})^{\mathrm{T}}$
We conclude by observing that $\Det Q(\bm{x}) = \prod_{i \in [N]} q(x_i) >0$, since $q$ is positive by assumption.

\subsubsection{Proof of~\Cref{prop:bound_tELS_DPP}}\label{proof:bound_tELS_DPP}


Let $\bm{w} =(w_{m})_{m \in [N]} \in \mathbb{R}^{N}$. We have
\begin{equation}
\|f-\sum\limits_{m\in[M]}w_{m}e_{m} \|_{\omega}^2 = \|f\|_{\omega}^2 - 2 \sum\limits_{m\in[M]}w_{m}\langle f,e_{m} \rangle_{\omega} + \sum\limits_{m\in[M]}w_{m}^2.
\end{equation} 
Therefore
\begin{equation}\label{eq:f_fELS_N0}
\|f-\hat{f}_{\mathrm{ELS},M,\bm{x}} \|_{\omega}^2 = \|f\|_{\omega}^2 - 2 \sum\limits_{m\in[M]}\hat{w}_{m}\langle f,e_{m} \rangle_{\omega} + \sum\limits_{m\in[M]}\hat{w}_{m}^2,
\end{equation}
where $\hat{\bm{w}} = \delta(\bm{x})^{-1}\gamma(\bm{x})$, with $\delta(\bm{x})$ and $\gamma(\bm{x})$ given by~\eqref{eq:gamma_and_delta}. By~\Cref{prop:ELS_is_PI}, we get 
\begin{equation}\label{eq:hat_w_equal_EZ}
\forall n \in [N], \:\: \hat{w}_{n} = \hat{I}_{n}^{\mathrm{QI}}(f).
\end{equation}
Based on~\eqref{eq:hat_w_equal_EZ}, ~\eqref{eq:ez_1} implies that
\begin{equation}\label{eq:ez_consequence_1}
\forall n \in [N], \:\: \mathbb{E}_{\DPP} \hat{w}_{n} = \langle f, e_{n} \rangle_{\omega},
\end{equation}
and \rev{the second line of~\eqref{eq:ez_1}} implies that
\begin{align}\label{eq:ez_consequence_2}
\forall n \in [N], \:\: \mathbb{E}_{\DPP}\hat{w}_{n}^2 &= \mathbb{E}_{\DPP}\hat{I}^{\mathrm{QI}}_{n}(f)^2 \nonumber\\
 & = I_{n}(f)^2 +2I_{n}(f)\mathbb{E}_{\DPP}(\hat{I}_{n}^{\mathrm{QI}}(f)-I_{n}(f)) +  \mathbb{E}_{\DPP}(\hat{I}^{\mathrm{QI}}_{n}(f) - I_{n}(f))^2 \nonumber\\
& = \langle f,e_{n} \rangle_{\omega}^2 +  \sum\limits_{m \geq N+1}\langle f, e_{m}\rangle_{\omega}^2.
\end{align}
Finally, combining~\eqref{eq:f_fELS_N0}, \eqref{eq:ez_consequence_1} and~\eqref{eq:ez_consequence_2}, we get
\begin{align*}
\mathbb{E}_{\DPP}\|f-\hat{f}_{\mathrm{ELS},M,\bm{x}} \|_{\omega}^2 & = \|f\|_{\omega}^2 - 2\sum\limits_{n\in [M]} \langle f,e_{n} \rangle_{\omega}^2 + \sum\limits_{n\in[M]} \big( \langle f,e_{n} \rangle_{\omega}^2 + \sum\limits_{m \geq N+1} \langle f, e_{m} \rangle_{\omega}^2 \big) \\
& = \sum\limits_{n \geq M+1} \langle f,e_{n}\rangle_{\omega}^2 + \sum\limits_{n \in [M]} \sum\limits_{m \geq N+1} \langle f, e_{m} \rangle_{\omega}^2 \\
& = \|f-f_{M}\|_{\omega}^2 + M\|f-f_{N}\|_{\omega}^2.
\end{align*}

\section{Numerical illustrations}\label{s:numsims}

\rev{This section illustrates the results of~\Cref{sec:new_results} for three families of RKHSs: periodic Sobolev spaces on $[0,1]$ in \cref{sec:sobolev_space}, an RKHS with a rotation-invariant kernel on the hypersphere in \cref{s:numsims_hypersphere} and an RKHS spanned by the uni-dimensional prolate spheroidal wave functions in \cref{s:numsims_unidim_PSWF}.}

\subsection{Periodic Sobolev spaces}\label{sec:sobolev_space}
This section illustrates the superconvergence phenomenon of \Cref{thm:useful_result_DPP} in a one-dimensional domain.
Let $\mathcal{X} = [0,1]$ be equipped with the uniform measure $\omega$, and define for $s \in \mathbb{N}^{*}$, for all $(x,y)\in \X \times \X$, the kernel
\begin{equation}
  \label{eq:periodic_sobolev_space_in_cossin}
  k_{s}(x,y) = 1+ 2\sum\limits_{m \in \mathbb{N}^{*}} \frac{1}{m^{2s}} \cos(2\pi m(x-y)),
\end{equation}
where the convergence holds uniformly on $\X \times \X$.
The kernel $k_{s}$ can be expressed in closed form using Bernoulli polynomials \citep{Wah90},
\begin{equation}
k_{s}(x,y) = 1 + \frac{(-1)^{s-1}(2 \pi)^{2s}}{(2s)!} B_{2s}(\{x-y\}).
\end{equation}
The corresponding RKHS $\mathcal{F}=\mathcal{S}_{s}$ is the periodic Sobolev space of order $s$. An element of $\mathcal{S}_{s}$ is a function $f$ defined on $[0,1]$, that has a derivative of order $s$ in the sense of distributions such that
$f^{(s)} \in \mathbb{L}_{2}(\omega)$, and
\begin{equation}
\forall i \in \{0, \dots, s-1\}, \:\: f^{(i)}(0) = f^{(i)}(1);
\end{equation}
see Chapter 7 of \citep{BeTh11}.
This class of RKHSs is ideal to validate the theoretical guarantees obtained in \Cref{sec:new_results}, since the eigenvalues $\sigma_m$ and the eigenfunctions $e_m$ are known explicitly. \rev{Moreover, $(k_{s})_{2} = k_{2s}$ for $s \in \mathbb{N}^{*}$, which is a practical identity when using \eqref{eq:w_hat_using_k2}.}

\rev{Figure~\ref{fig2} shows log-log plots of $\|f-\hat{f}_{\mathrm{LS},\bm{x}}\|_{\omega}^2$ vs. $N$, when $f = e_{m}^{\F}$ with $m \in \{1,2,3,4,5\}$. 
Each point is an average over $50$ independent draws from the projection DPP.
$\F$ is the periodic Sobolev spaces of order $s=1$ on figure~\ref{fig:optimal_em_l2_ls_sobolev_s_1} and $s=2$ on figure~\ref{fig:optimal_em_l2_ls_sobolev_s_2}.} 
We observe that the expected squared residual converges to $0$ at the same rate as $\sigma_{N+1}^{2} = \mathcal{O}(N^{-4s})$, which is slightly faster than the rate of convergence of $\sum_{m \geq N+1} \sigma^{2}_{m} = \mathcal{O}(N^{1-4s})$ predicted by~\Cref{thm:useful_result_DPP}. 

\rev{Figure~\ref{fig:optimal_em_l2_oka_sobolev_s_1} and Figure~\ref{fig:optimal_em_l2_oka_sobolev_s_2} illustrate this fast rate of convergence for the kernel-based interpolant $\hat{f}_{\mathrm{OKA},\bm{x}}$ as well.} Indeed, these figures show log-log plots of $\|f-\hat{f}_{\mathrm{OKA},\bm{x}}\|_{\omega}^2$ when $f = e_{m}^{\F}$ with $m \in \{1,2,3,4,5\}$ w.r.t. $N$, averaged over $50$ independent DPP samples. The squared residuals are evaluated using the formula~\eqref{eq:identity_epsilon_f}.

\begin{figure}
    \centering
\subfloat[The residual $\|f-\hat{f}_{\mathrm{LS},\bm{x}}\|_{\omega}^2$ vs. $N$ ($s = 1$).\label{fig:optimal_em_l2_ls_sobolev_s_1} ]{%
      \includegraphics[width=0.47\textwidth]{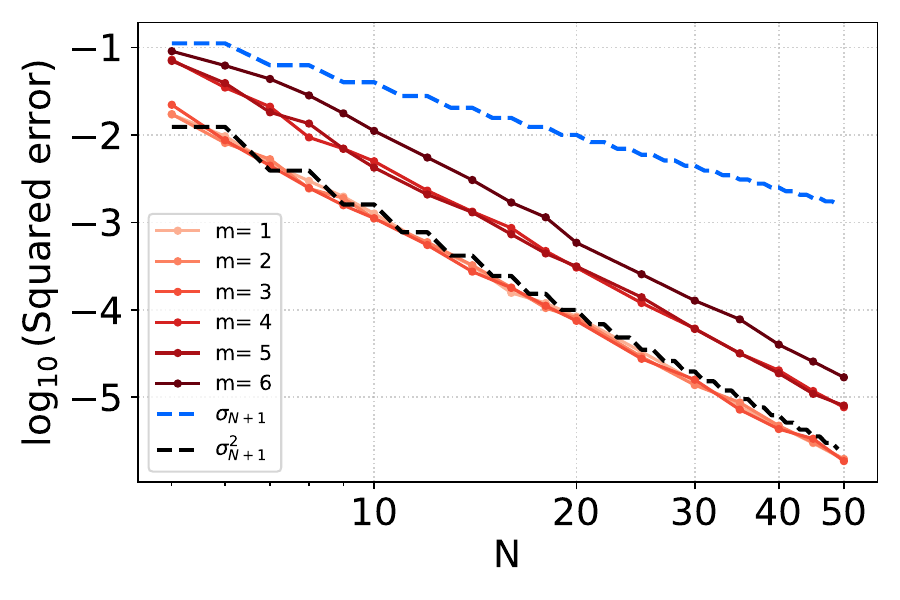}
    }~\subfloat[The residual $\|f-\hat{f}_{\mathrm{LS},\bm{x}}\|_{\omega}^2$ vs. $N$ ($s = 2$).\label{fig:optimal_em_l2_ls_sobolev_s_2} ]{%
      \includegraphics[width=0.47\textwidth]{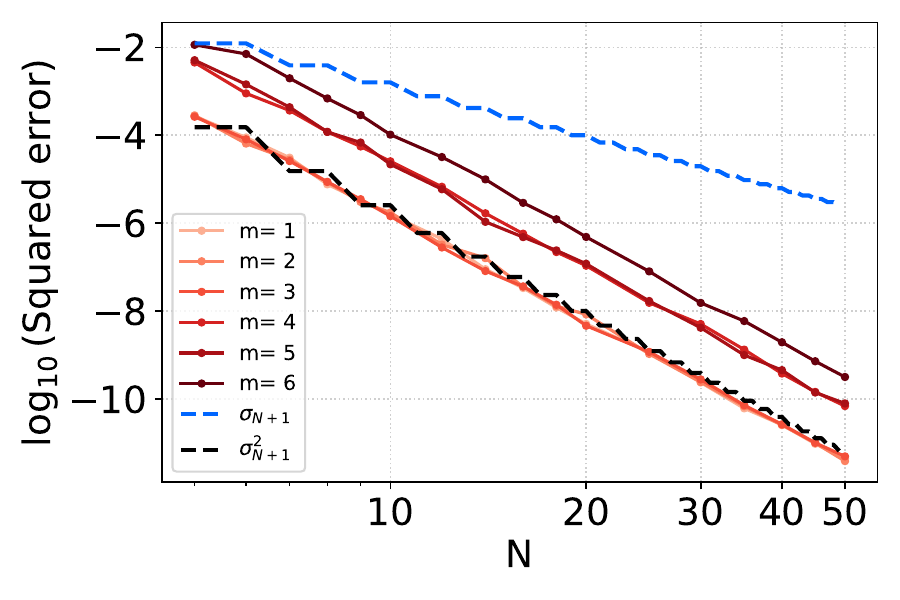}
    }\\

\subfloat[The residual $\|f-\hat{f}_{\mathrm{OKA},\bm{x}}\|_{\omega}^2$ vs. $N$ ($s = 1$).\label{fig:optimal_em_l2_oka_sobolev_s_1}]{%
      \includegraphics[width=0.47\textwidth]{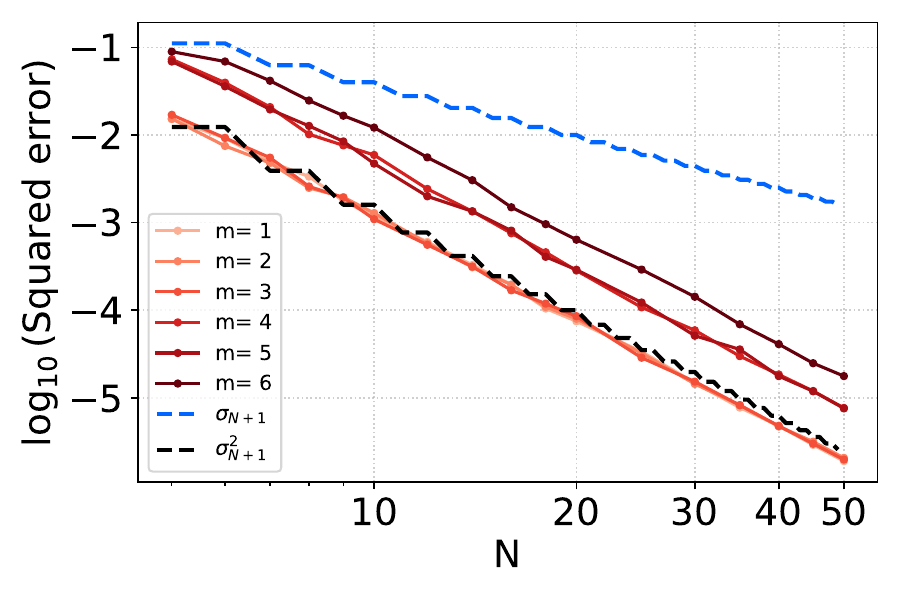}
    }~\subfloat[The residual $\|f-\hat{f}_{\mathrm{OKA},\bm{x}}\|_{\omega}^2$ vs. $N$ ($s = 2$).\label{fig:optimal_em_l2_oka_sobolev_s_2}]{%
      \includegraphics[width=0.47\textwidth]{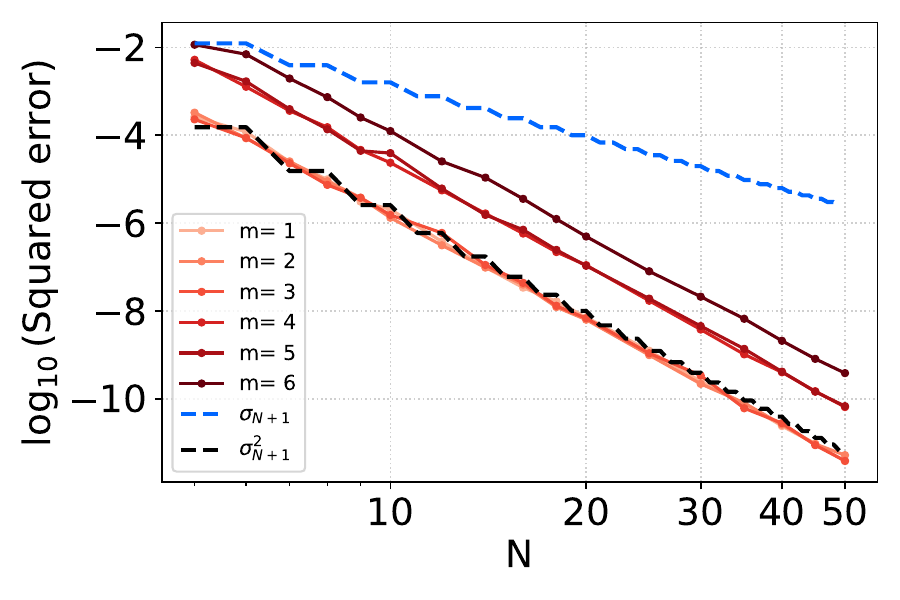}
    }\\
\caption{The reconstruction error for $f = e_{m}^{\mathcal{F}}$, when $\F$ is the periodic Sobolev space of order $s \in \{1,2\}$. \label{fig2}}
\end{figure}

\rev{Figure~\ref{fig3} considers the reconstruction of a random function}
\begin{equation}\label{eq:f_random_xi}
f = \sum\limits_{m \in [M]}\xi_{m} e_{m}^{\F},
\end{equation}
where $M \in \mathbb{N}^{*}$ and the $\xi_{m}$ are i.i.d. standard Gaussians. 
Each function $f$ is a random element of the eigenspace $\mathcal{E}^{\F}_{M}$.
Figure~\ref{fig:superconvergence_l2_ls_sobolev_s_1} and Figure~\ref{fig:superconvergence_l2_ls_sobolev_s_2} show log-log plots of $\|f-\hat{f}_{\mathrm{LS},\bm{x}}\|_{\omega}^2$ w.r.t. $N$, averaged over 50 independent DPP samples, and for a set of $50$ functions $f$ sampled according to~\eqref{eq:f_random_xi}, when $\F$ is the periodic Sobolev spaces of order $s=1$. 
Different DPP samples are used for each function $f$. 
The empirical convergence rate scales as $\mathcal{O}(\sigma_{N+1}^{2})$, which is slightly faster than $\mathcal{O}(\sum_{m \geq N+1} \sigma^{2}_{m})$. Again, this fast convergence is also observed for the kernel-based interpolant $\hat{f}_{\mathrm{OKA},\bm{x}}$, as shown in Figure~\ref{fig:superconvergence_l2_oka_sobolev_s_1} and Figure~\ref{fig:superconvergence_l2_oka_sobolev_s_2}. 



\begin{figure}[h!]
    \centering
\subfloat[The residual $\|f-\hat{f}_{\mathrm{LS},\bm{x}}\|_{\omega}^2$ vs. $N$ for $M=10$.\label{fig:superconvergence_l2_ls_sobolev_s_1}]{%
      \includegraphics[width=0.47\textwidth]{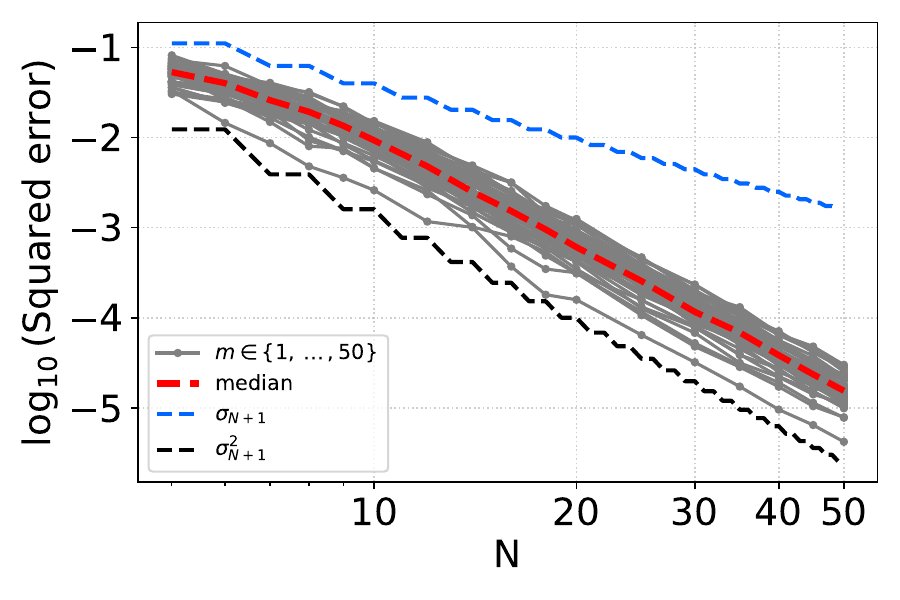}
    }~\subfloat[The residual $\|f-\hat{f}_{\mathrm{LS},\bm{x}}\|_{\omega}^2$ vs. $N$ for $M=20$.\label{fig:superconvergence_l2_ls_sobolev_s_2}]{%
      \includegraphics[width=0.47\textwidth]{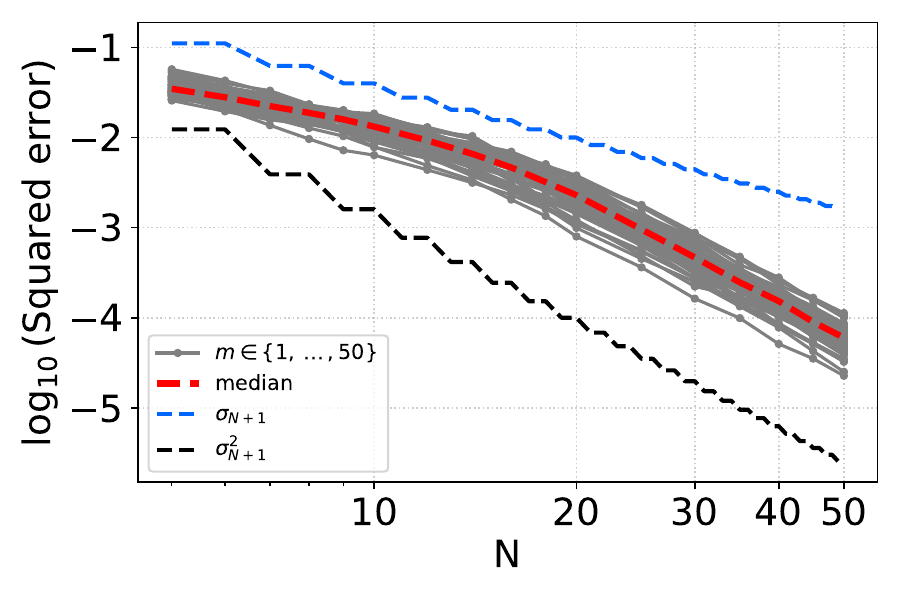}
    }\\
\subfloat[The residual $\|f-\hat{f}_{\mathrm{OKA},\bm{x}}\|_{\omega}^2$ vs. $N$ for $M=10$.\label{fig:superconvergence_l2_oka_sobolev_s_1}]{%
      \includegraphics[width=0.47\textwidth]{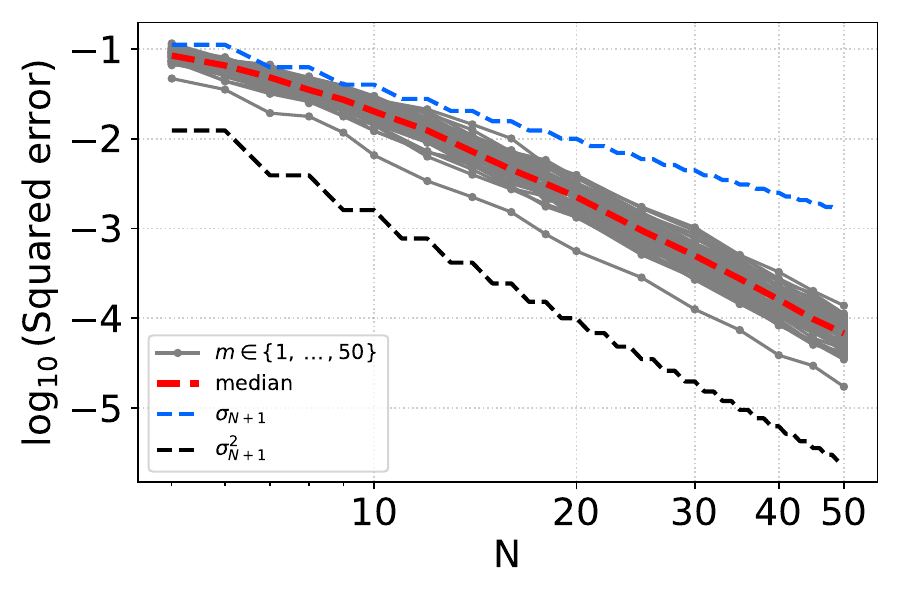}
    }~\subfloat[The residual $\|f-\hat{f}_{\mathrm{OKA},\bm{x}}\|_{\omega}^2$ vs. $N$ for $M=20$.\label{fig:superconvergence_l2_oka_sobolev_s_2}]{%
      \includegraphics[width=0.47\textwidth]{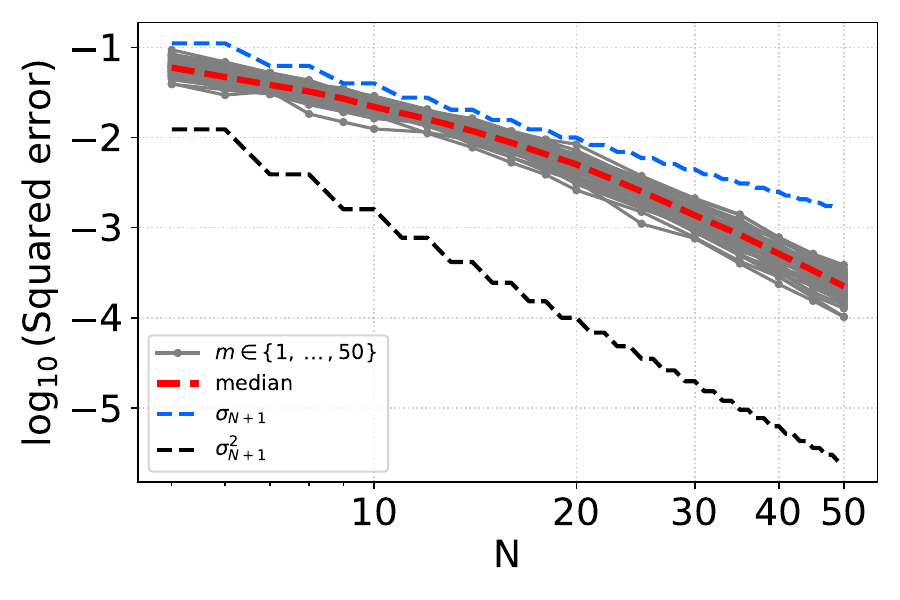}
    }\\
\caption{The reconstruction error for \rev{$50$} samples of $f$ from the distribution defined by~\eqref{eq:f_random_xi} when $\F$ is the periodic Sobolev space of order $s =1$.  \label{fig3} 
}
\end{figure}

\subsection{An RKHS with a rotation-invariant kernel on the hypersphere}
\label{s:numsims_hypersphere}

\begin{figure}[h!]
    \centering
\subfloat[The residual $\|f-\hat{f}_{\mathrm{LS},\bm{x}}\|_{\omega}^2$ vs. $N$ ($s = 1$).\label{fig:optimal_em_l2_ls_sobolev_sphere_s_1}]{%
      \includegraphics[width=0.47\textwidth]{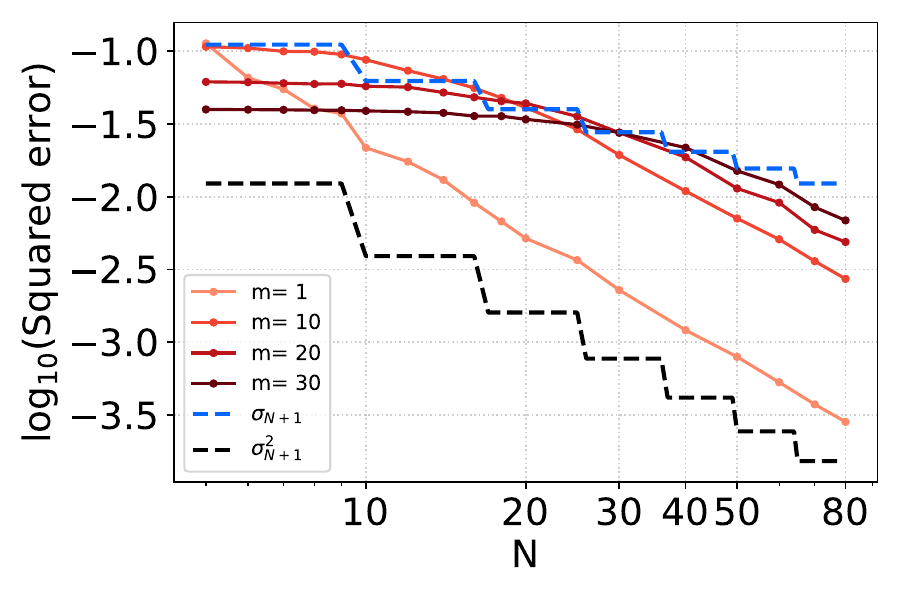}
    }~\subfloat[The residual $\|f-\hat{f}_{\mathrm{LS},\bm{x}}\|_{\omega}^2$ vs. $N$ ($s = 2$).\label{fig:optimal_em_l2_ls_sobolev_sphere_s_2}]{%
      \includegraphics[width=0.47\textwidth]{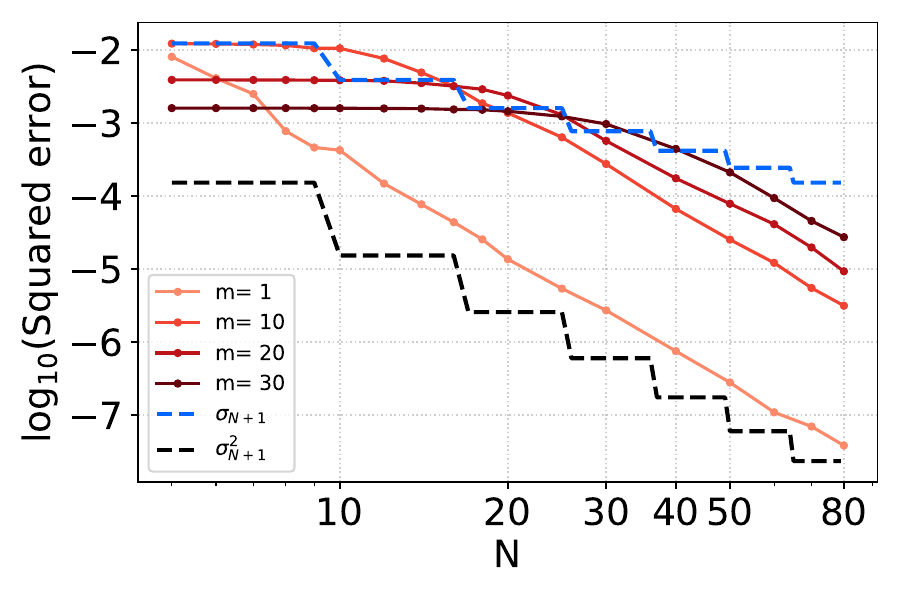}
    }\\

\caption{The reconstruction error for $f = e_{m}^{\mathcal{F}}$, when $\F$ is the periodic Sobolev space of order $s \in \{1,2\}$ in the hypersphere $\mathbb{S}^{d-1}$ where $d=3$. \label{fig4}}
\end{figure}

\rev{The RKHS framework permits to work in high-dimensional and/or non-Euclidean domains.
This section illustrates the superconvergence phenomenon of \Cref{thm:useful_result_DPP} on a hypersphere.} 
The definition of a positive definite kernel on a hypersphere dates back to the seminal work of \cite{Sch42}. 
A kernel $k:\mathbb{S}^{d-1} \times \mathbb{S}^{d-1} \rightarrow \mathbb{R}$ is said to be a \emph{dot-product kernel} if there
exists a function $\varphi: [-1,1] \rightarrow \mathbb{R}$ such that $k(x,y) = \varphi(\langle x, y \rangle)$ for $x, y \in \mathbb{S}^{d-1}$. 
This defines a large class of rotation-invariant kernels on $\mathbb{S}^{d-1}$. 
Moreover, the integration operator $\bm{\Sigma}$ associated to a dot-product kernel decomposes in the basis of spherical harmonics, when $\mathbb{S}^{d-1}$ is equipped with the uniform measure. Then Mercer's decomposition holds in the form
\begin{equation}
  \label{eq:sphere_kernel}
  k(x,y) = \sum\limits_{\ell \in \mathbb{N}} \sigma_{\ell} \sum\limits_{i\in [N(d,\ell)]} Y_{\ell,i}(x)Y_{\ell,i}(y),
\end{equation}
where for $\ell \in \mathbb{N}$, $\{Y_{\ell,i}: \mathbb{S}^{d-1} \rightarrow \mathbb{R}, i=1,\dots, N(d,\ell) \}$ is the basis of spherical harmonics of exact degree $\ell$ and $N(d,\ell) := (2\ell +d -1) \Gamma(\ell + d -1)/\big(\Gamma(d)\Gamma(\ell+1)\big)$; see \citep{Gro96}.
The exact expression of $\sigma_{\ell}$ or the eigenvalue decay may be found in \citep{CuFr97,SmOvWi00,Bac17b,AzMe14,ScHa21}.



\rev{The following set of experiments uses the kernel obtained by taking $d=3$ and  $\sigma_{\ell} = (1+\ell)^{-2s}$ in \eqref{eq:sphere_kernel} to illustrate the superconvergence phenomenon of \Cref{thm:useful_result_DPP}.}
The corresponding RKHS is akin to a Sobolev space of order $s$. An element of $\mathcal{F}$ is a function $f$ defined on $\mathbb{S}^{d-1}$, which has a derivative of order $s$ in the sense of distributions such that
$f^{(s)} \in \mathbb{L}_{2}(\omega)$ \citep{Hes06}. \rev{In this case, the kernel $k_2$ does not have an explicit formula. For this reason, the value of $k_{2}(x,y)$ is numerically approximated by truncating the sum to a sufficiently large order. }

\rev{Figure~\ref{fig4} shows log-log plots of $\|f-\hat{f}_{\mathrm{LS},\bm{x}}\|_{\omega}^2$ when $f = e_{m}^{\F}$ with $m \in \{1,10,20,30\}$ w.r.t. $N$, averaged over 50 independent DPP samples, figure~\ref{fig:optimal_em_l2_ls_sobolev_sphere_s_1} for $s=1$ and figure~\ref{fig:optimal_em_l2_ls_sobolev_sphere_s_2} for $s=2$.} 
Again, the expected squared residual converges to $0$ at the same rate as $\sigma_{N+1}^{2} = \mathcal{O}(N^{-4s})$, which is slightly faster than the rate of convergence predicted by~\Cref{thm:useful_result_DPP} of $\sum_{m \geq N+1} \sigma^{2}_{m} = \mathcal{O}(N^{1-4s})$. 
As predicted by~\Cref{thm:useful_result_DPP}, the superconvergence regime corresponds to $N \geq m$.






\subsection{The RKHS spanned by the uni-dimensional PSW functions}\label{s:numsims_unidim_PSWF}
This section compares DPP-based sampling with i.i.d. Christoffel sampling in the classical setting of band-limited signals. 

Let now $\mathcal{X} = [-T/2,T/2]$ equipped with $\omega$ the uniform measure, and let $F>0$.
Consider the so-called \emph{Sinc} kernel
\begin{equation}\label{eq:sinc_F_kernel}
  k_{F}(x,y):= \mathrm{Sinc}(F(x-y)) = \frac{\sin(F(x-y))}{F(x-y)}.
\end{equation}
This kernel defines an RKHS that corresponds to the space of band-limited functions restricted to the interval $\mathcal{X}$.
Slepian, Landau, and Pollak proved that the eigenfunctions of the integration operator $\bm{\Sigma}$ associated to~\eqref{eq:sinc_F_kernel} and the measure~$\omega$ satisfy a differential equation known in physics as the prolate spheroidal wave (PSW) equation; see \citep{SlPo61}. 
Since this seminal work, the eponymous functions were subject to extensive research. 
A detailed description of the eigenfunctions was carried out in \citep{Osi13,BoKa14,OsRo12}. 
In particular, these functions were shown to be well represented in the orthonormal basis defined by the Legendre polynomials \citep{Boy05,OpRoXi13}.  

\begin{figure}[h!]
 \begin{center}
 \subfloat[The residual $\|f-\hat{f}_{\mathrm{OKA},\bm{x}}\|_{\omega}^2$ vs. $N$ ($m = 1$).]{\includegraphics[width=0.43\textwidth]{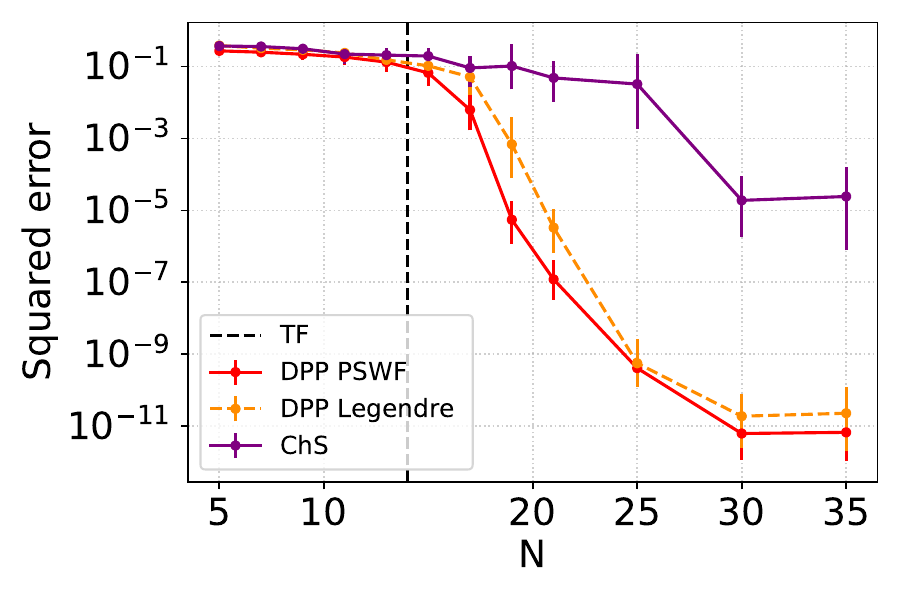}}~\subfloat[The residual $\|f-\hat{f}_{\mathrm{OKA},\bm{x}}\|_{\omega}^2$ vs. $N$ ($m = 2$).]{\includegraphics[width=0.43\textwidth]{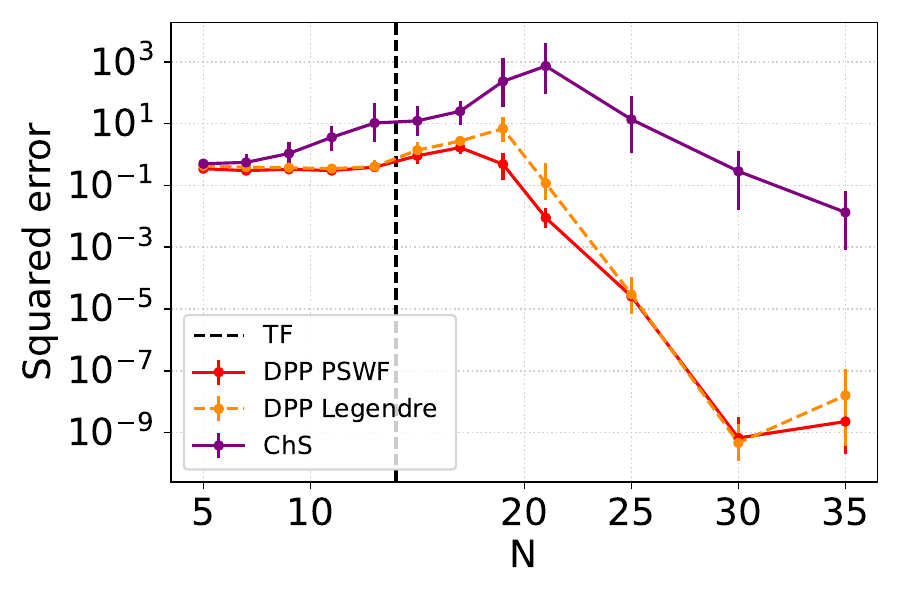}}\\
 \subfloat[The residual $\|f-\hat{f}_{\mathrm{OKA},\bm{x}}\|_{\omega}^2$ vs. $N$ ($m = 3$).]{\includegraphics[width=0.43\textwidth]{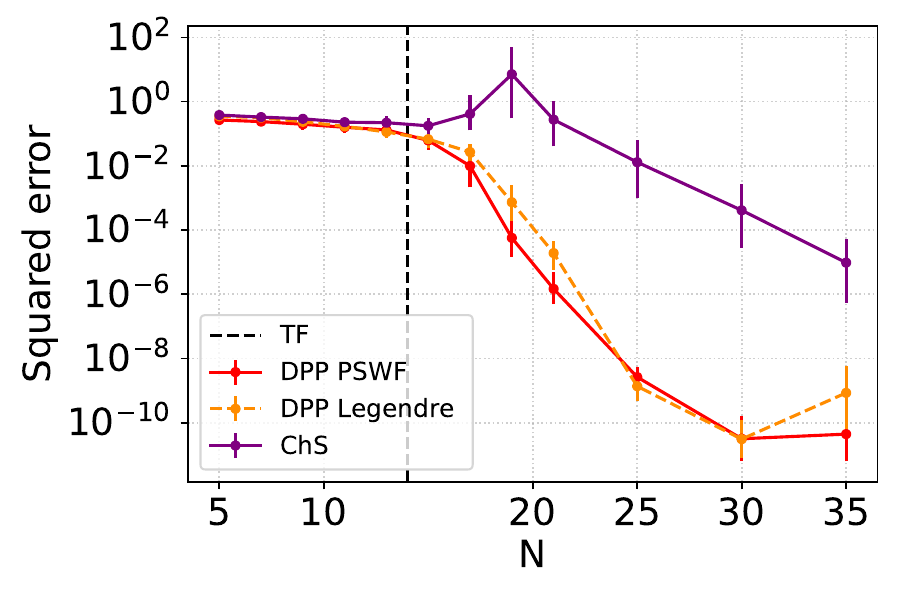}}~\subfloat[The residual $\|f-\hat{f}_{\mathrm{OKA},\bm{x}}\|_{\omega}^2$ vs. $N$ ($m = 4$).]{\includegraphics[width=0.43\textwidth]{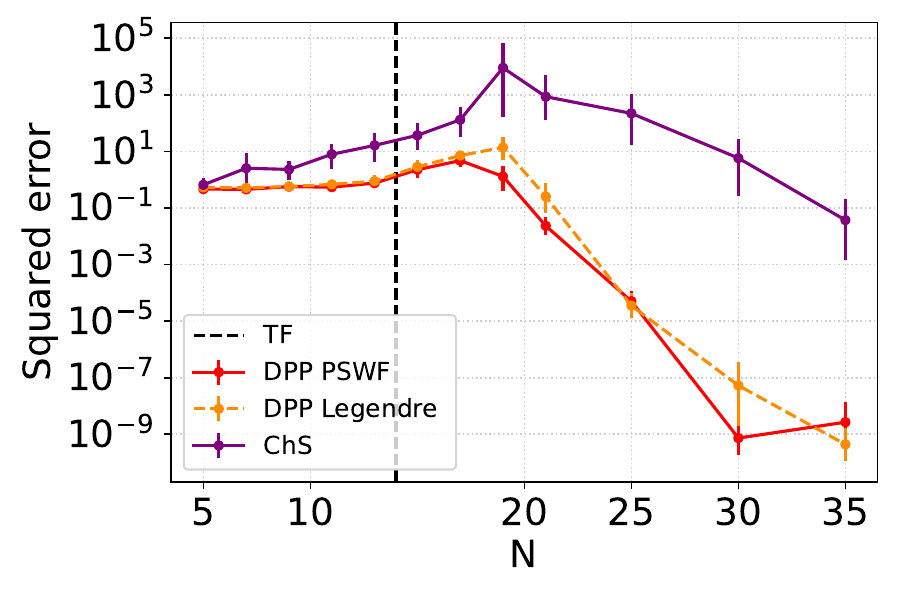}}\\
 \end{center}
\caption{The reconstruction error for $f = e_{m}^{\mathcal{F}}$ in the RKHS associated to the Sinc kernel. \label{fig:pswf}}
\end{figure}

The asymptotics of the eigenvalues $(\sigma_m)$ of $\bm{\Sigma}$ in the limit $c:=TF \rightarrow +\infty$ were investigated too \citep{LaWi80}: for $\epsilon>0$, in this asymptotic regime, $\bm{\Sigma}$ has approximately $TF$ eigenvalues in the interval $[1-\epsilon,1]$, $\mathcal{O}(\log(TF) \log(1/\epsilon)) $ eigenvalues in the interval $[\epsilon,1-\epsilon]$, and the remaining eigenvalues decrease to $0$ at an exponential rate.

This set of experiments studies the influence of the design $\bm{x}$ on the convergence of $\hat{f}_{\mathrm{OKA}, \bm{x}}$ to $f$ with respect to the norm $\|.\|_{\omega}$. 
The following random designs are compared: \emph{(i)} the projection DPP defined in \Cref{s:DPPs} associated to the first PSW functions (DPP-PSWF), \emph{(ii)} the projection DPP associated to the first normalized Legendre polynomials (DPP-Legendre), and \emph{(iii)} i.i.d. Christoffel sampling (ChS). \rev{Note that, although i.i.d Christoffel sampling is standard for ELS, it comes with no
guarantees for OKA-based reconstruction.}
Figure~\ref{fig:pswf} shows log-log plots of $\|f-\hat{f}_{\mathrm{OKA},\bm{x}}\|_{\omega}^2$ w.r.t. $N$, when $f = e_{m}^{\F}$ with $m \in \{1,2,3,4\}$, averaged over 50 independent samples of each of the three distributions. 
Both DPP-PSWF and DPP-Legendre significantly improve over Christoffel sampling. 
Moreover, the expected squared residual $\|f-\hat{f}_{\mathrm{OKA},\bm{x}}\|_{\omega}^2$ under the two DPPs converges to $0$ at an exponential rate when $N \geq TF$: taking $T = 2$ and $F = 7$ so that $TF = 14$
which corresponds to the asymptotics described in \citep{LaWi80}.

\section{Discussion and open questions}
\label{s:discussion}

This section gathers some high-level comments and discusses possible extensions of our results. 






First, we insist that generating different random designs requires access to different quantities.
Christoffel sampling in general will require rejection sampling. 
One should thus be able to evaluate the density $x \mapsto \sum_{m \in [N]} e_{m}(x)^2/N$, \rev{and to find a suitable proposal distribution, in order to control the number of rejections.} The study of the `shape' of the Christoffel function is thus crucial for sampling, and is an active topic of research \citep{PaBaVe18,AvKaMuVeZa19,DoCo22c}. 
A good sampler for the Christoffel function is also relevant to simulate the DPP of \eqref{def:density_detsampling_T} \rev{with $T=\{1,\dots, N\}$}, as the p.d.f. of Christoffel sampling can be used as a proposal distribution when sampling the sequential conditionals of the HKPV algorithm \citep{HoKrPeVi06}; see \citep{GaBaVa19}.  
When the p.d.f. of Christoffel sampling cannot be evaluated, one can still resort to continuous volume sampling. 
Indeed, while the only known \emph{exact} sampling algorithm for CVS still relies on possibly hard-to-evaluate projection kernels, there are approximate samplers that leverage the fact that evaluating the p.d.f. of CVS in \Cref{def:density_vs} only requires evaluating the RKHS kernel $k$.
In particular, \cite{ReGh19} have studied a natural Markov chain Monte Carlo sampler, whose mixing time scales as $\mathcal{O}(N^{5}\log(N))$.
Interestingly, each iteration of their Markov chain requires a rejection sampling step, the expected number of rejections is shown to be $\mathcal{O}(1/\sum_{m \geq N+1} \sigma_m)$. 
In other words, the smoother the kernel, the harder it is to run the MCMC algorithm. 
It would be interesting to investigate whether alternative MCMC algorithms can circumvent this `smoothness curse'.










Second, we comment on the instance optimal property (IOP), which motivated the introduction of Christoffel sampling \citep{CoMi17}. 
In particular, it was proven that for some variants of the empirical least-squares approximation, the IOP holds when the sampling budget \rev{$N$} is as low as $\mathcal{O}(M \log(M))$. These variants essentially exclude configurations of nodes $\bm{x}$ for which the Gram matrix $\bm{G}_{q,\bm{x}}$, defined in the end of~\Cref{sec:eigenspace_approximations}, is ill-conditioned. 
Similarly, determinantal sampling implicitly favors configurations of nodes $\bm{x}$ so that $\Det \bm{G}_{q,\bm{x}}$ is large. 
Moreover, \Cref{prop:bound_tELS_DPP} shows that the IOP holds under a suitable projection DPP, with a minimum sampling budget $N \geq M$. 
The price to pay is the constant $M+1$ in the IOP. 
Now, when $N=M$, this constant is $N+1$, which actually improves upon the constant $N^2+1$ proven in \citep{ChDo23} for an algorithm based on the so-called \emph{effective resistances}. 
The latter algorithm generates randomized configurations $\bm{x}$ in a greedy fashion so that the `redundancy' of sampling is reduced. 

Third, we might seek approximation schemes that are optimal in some worst-case sense, rather than looking for ones that satisfy the IOP. 
This is the approach adopted in~\citep{KrUl21,DoKrUl23}, for instance, where the authors investigate the \emph{sampling numbers}
\begin{equation}
g_N:= \inf\limits_{\substack{x_{1},\dots,x_{N} \in \mathcal{X}\\\varphi_1, \dots,\varphi_{N} \in \mathbb{L}_{2}(\omega)}}\sup\limits_{\substack{f \in \mathcal{F}\\\|f\|_{\mathcal{F}} \leq 1}}\|f - \sum\limits_{i \in [N]} f(x_i)\varphi_i\|_{\omega}.
\end{equation}
The sampling numbers measure the complexity of the simultaneous recovery of all functions in the unit ball of $\mathcal{F}$. A consequence of the work of \cite{DoKrUl23} is that there is a universal constant $c \in \mathbb{N}^{*}$ such that $g_{cN} = \mathcal{O}(\sigma_{N+1})$ under assumptions on the decay of $(\sigma_m)$; see also \citep{KrUl21}.
\rev{It would be interesting to investigate whether the expected worst-case approximation error under a determinantal distribution matches the optimal rate of convergence.} 
Such a worst-case optimality would further connect to a string of results on the \emph{completeness of DPPs} \citep{Lyo03,Gho15,BuQiSh21}, which look for conditions under which a sample $\bm{x}$ from a DPP is a \emph{uniqueness set} in the sense that two
elements of $\mathcal{F}$ that coincide on $\bm{x}$ are equal.

\section{Conclusion}
\label{s:conclusion}

This article studies various finite-dimensional approximations for functions living in an RKHS, based on a finite number of repulsive nodes defined by RKHS-adapted determinantal distributions.
Our results, \rev{summarized in \cref{tab:summary},} give convergence guarantees in $L^2$ norm for \emph{any} function living in the RKHS. They are stronger than those of previous works on determinantal sampling that give convergence guarantees in RKHS norm for functions living in a strict subspace of the RKHS. 
In particular, we have shown that the mean square error of the least-squares approximation converges to zero at a rate that depends on the eigenvalues of the Mercer decomposition of the RKHS kernel.
Moreover, we show that the convergence rate is faster for functions living in a certain low-dimensional subspace of the RKHS. 
This result provides insight on how the rate improves with smoothness. \rev{In addition}, we have investigated approximations that can be evaluated using only a finite number of evaluations of the target function, unlike the least-squares approximation. 
In particular, we prove that the instance optimality property holds for a particular transform-based approximation using DPPs, even with a minimal sampling budget. 
This result shows that determinantal sampling generalizes and improves on i.i.d. sampling from the Christoffel function. 
The numerical experiments conducted in various domains validate our theoretical results. The numerical performance actually slightly exceeds the expectations. This observation both illustrates the relevance of the proposed approach and suggests that there is room for future work to tentatively improve our theoretical results.

\section*{Acknowledgments}

AB acknowledges support from the AllegroAssai ANR project ANR-19-CHIA-0009. RB acknowledges support from the ERC grant BLACKJACK ERC-2019-STG-851866 and the Baccarat ANR project ANR-20-CHIA-0002. PC acknowledges support from the Sherlock ANR project ANR-20-CHIA-0031-01, the programme
d’investissements d’avenir ANR-16-IDEX-0004 ULNE and Région Hauts-de-France.





%

\newpage
\bibliographystyle{plainnat}
\bibliography{bibliography.bib}

\end{document}